\documentclass[lettersize,journal]{IEEEtran}
\usepackage{macro}
\usepackage{amsmath,amsfonts}
\usepackage{amssymb}
\usepackage{amsthm}
\usepackage{algorithm}
\usepackage[noend]{algpseudocode}
\usepackage{array}
\usepackage[caption=false,font=normalsize,labelfont=sf,textfont=sf]{subfig}
\usepackage{textcomp}
\usepackage{stfloats}
\usepackage{url}
\usepackage{verbatim}
\usepackage{graphicx}
\usepackage{cite}
\usepackage{hyperref}
\usepackage{color}
\usepackage{multirow}
\newtheorem{Thm}{Theorem}
\newtheorem{Def}{Definition}

\newtheorem{Ass}{Assumption}
\newtheorem{Rem}{Remark}

\newtheorem{Pro}{Proposition}

\newcommand{\alg}{DP-CSGP}

\begin{document}
\title{DP-CSGP: Differentially Private Stochastic Gradient Push with Compressed Communication}

\author{
Zehan Zhu, Heng Zhao, Yan Huang, Joey Tianyi Zhou, Shouling Ji and Jinming Xu$^\dagger$
\thanks{Z. Zhu, S. Ji, and J. Xu are with Zhejiang University, China. Y. Huang is with KTH Royal Institute of Technology, Sweden. H. Zhao and J.T. Zhou are with the A*STAR Centre for Frontier AI Research (CFAR), Singapore.
}
\thanks{$^\dagger$Correspondence to jimmyxu@zju.edu.cn (Jinming Xu).}
}




\maketitle
\begin{abstract}

In this paper, we propose a \underline{D}ifferentially \underline{P}rivate \underline{S}tochastic \underline{G}radient \underline{P}ush with \underline{C}ompressed communication (termed {\alg}) for decentralized learning over directed graphs. Different from existing works, the proposed algorithm is designed to maintain high model utility while ensuring both rigorous differential privacy (DP) guarantees and efficient communication.
For general non-convex and smooth objective functions, we show that the proposed algorithm achieves a tight utility bound of $\mathcal{O}\left( \sqrt{d\log \left( \frac{1}{\delta} \right)}/(\sqrt{n}J\epsilon) \right)$ ($J$ and $d$ are the number of local samples and the dimension of decision variables, respectively) with $\left(\epsilon, \delta\right)$-DP guarantee for each node, matching that of decentralized counterparts with exact communication.
Extensive experiments on benchmark tasks show that, under the same privacy budget, {\alg} achieves comparable model accuracy with significantly lower communication cost than existing decentralized counterparts with exact communication.

\end{abstract}


\section{Introduction}
\IEEEPARstart{D}{istributed} learning has been widely adopted in various application domains due to its great potential in improving computing efficiency~\cite{langer2020distributed}.
In particular, we assume that each computing node has $J$ data samples, and we use $f_i(x; j)$ to denote the loss of the $j$-th data sample at node $i$ with respect to the model parameter $x\in \mathbb{R}^d$. We are then interested in solving the following non-convex finite-sum optimization problem via a group of $n$ nodes:
\begin{equation}
\label{global_loss_function}
\underset{x\in \mathbb{R}^d}{\min}f\left( x \right) \triangleq \frac{1}{n}\sum_{i=1}^n{f_i\left( x \right)},
\end{equation}
where $f_i\left( x \right) \triangleq \frac{1}{J}\sum_{j=1}^J{f_i\left( x;j \right)}$ is the loss function of node $i$ and all nodes collaborate to find a common model parameter $x$ minimizing their average loss functions.
We also assume that each node $i$ can only evaluate local stochastic gradient $\nabla f_i\left( x;\xi_i \right)$, $\xi_i \in \{1,2,...,J\}$.

For distributed parallel methods where there is a central coordinator such as a parameter server~\cite{li2014scaling,mcmahan2017communication}, they suffer from high communication overhead at the central server and expose the system to a single point of failure~\cite{lian2017can}. These limitations have stimulated growing interest in fully decentralized approaches~\cite{lian2017can,lian2018asynchronous,zhu2024r} for addressing Problem~\eqref{global_loss_function}, where no central node is required and each node communicates only with its neighbors. The existing decentralized learning algorithms usually employ undirected graphs for communication, which are difficult to implement in practice due to the existence of deadlocks~\cite{assran2019stochastic}. This motivates the need to consider more realistic scenarios in which communication links may be directed. Stochastic Gradient Push (SGP)~\cite{assran2019stochastic}, built upon the push-sum protocol~\cite{kempe2003gossip}, has been shown to be very effective for solving Problem~\eqref{global_loss_function} over directed communication networks.

With the continuous growth in the scale of modern deep learning models, the volume of data transmitted by each node at each iteration also increases. Under bandwidth-limited or heterogeneous network conditions, such high communication overhead becomes a critical bottleneck that restricts the efficiency of decentralized learning. The need for communication efficiency has motivated the development of communication compression techniques, which can significantly reduce the communication burden per iteration by transferring compressed information. Consequently, a series of recent works have explored decentralized learning methods that incorporate communication compression mechanism, including but not limited to~\cite{koloskova2019decentralized,singh2021squarm,zhao2022beer,huang2024cedas,yan2023compressed}.

In addition, it has been well known that the frequent exchange of model parameters in decentralized learning may lead to severe concern on privacy leakage as the disclose of intermediate parameters could potentially compromise the original data~\cite{wang2019beyond}. For instance, previous studies~\cite{truex2019hybrid,carlini2019secret} have shown that the exposed parameters can be utilized to crack original data samples. To ensure the training process does not accidentally leak private information, differential privacy (DP), as a theoretical tool to provide rigorous privacy guarantees and quantify privacy loss, have been widely integrated into training algorithms. A notable example is Abadi \textit{et al.}~\cite{abadi2016deep}, which developed a differentially private stochastic gradient descent (SGD) algorithm DP-SGD in the centralized (single-node) setting. Further, several differentially private algorithms~\cite{zhou2023optimizing, wei2021user, truex2020ldp, lowy2023private} are proposed for distributed (n-node) setting suitable for server-client architecture. More recently, DP has been incorporated into each node in fully decentralized learning systems to enhance the privacy protection~\cite{yu2021decentralized, xu2021dp, zhu2024privsgp}.

In this paper, we aim to develop a differentially private decentralized learning method with communication compression, to ensure both privacy guarantee and communication efficiency. Our main contributions are summarized as follows:
\begin{itemize}
\item \textbf{Communication-efficient algorithm with DP guarantee for each node.} Different from the existing works, we propose a communication-efficient differentially private decentralized learning method (termed {\alg}), which can work over general directed communication topologies in fully decentralized settings. The proposed {\alg} can ensure $(\epsilon,\delta)$-DP guarantee for each node, and significantly reduce the communication overhead per iteration thanks to the introduced error-feedback-based communication compression mechanism.

\item \textbf{Tight utility bound.} 
Given certain privacy budget $(\epsilon, \delta)$ for each node, we establish a tight utility bound of $\mathcal{O}\left( \sqrt{d\log \left( \frac{1}{\delta} \right)}/(\sqrt{n}J\epsilon) \right)$ for {\alg} under mild assumptions. The derived utility bound matches that of existing differentially private decentralized learning methods with exact communication, showing that our {\alg} maintains strong model utility while ensuring both DP guarantee and communication efficiency. To our knowledge, this is the first provably tight model utility guarantee for communication-compressed and differentially private decentralized learning.

\item \textbf{Extensive experimental evaluations.} 
We conduct extensive experiments on two non-convex training tasks in fully decentralized settings, to evaluate the performance of the proposed {\alg}. The experimental results show that, under the same privacy budget, our {\alg} achieves comparable final model accuracy with significantly lower communication cost than existing decentralized counterparts with exact communication, which validates the communication efficiency of {\alg}.
\end{itemize}

\section{Related Works}
\subsection{Communication-compressed decentralized learning}
To improve communication efficiency, gradient compression techniques have been incorporated into distributed learning with server-client architecture in~\cite{de2015taming,alistarh2017qsgd}, where client nodes send compressed gradients to a central server for aggregation. However, the large variance of compressed gradients leads to the severe degradation of convergence rate and final model accuracy. Seide \textit{et al.}~\cite{seide20141} first introduce an error-feedback mechanism to compensate for the variance incurred by compression, achieving notable improvements in both theoretical guarantees and empirical performance. Similar mechanisms are adopted in~\cite{stich2018sparsified, alistarh2018convergence, mishchenko2024distributed, li2020acceleration, gorbunov2021marina, li2021canita} to enhance the convergence for server-client distributed learning methods. Furthermore, the authors in \cite{richtarik2021ef21, fatkhullin2021ef21} formalize the error-feedback mechanism and establish sublinear convergence rates for general non-convex and smooth objectives. 

Recently, the works in~\cite{tang2018communication, koloskova2019decentralized, koloskova2020decentralized, singh2021squarm, liao2023linearly, zhao2024faster} have extended communication compression schemes to decentralized settings, and establish the convergence rates respectively. Tang \textit{et al.}~\cite{tang2018communication} propose a decentralized learning method with communication compression and provide theoretical convergence guarantees for general non-convex objectives. However, their results are limited to unbiased compressors, restricting the method's practical applicability. Koloskova \textit{et al.}~\cite{koloskova2019decentralized} incorporate the error-feedback–based communication compression technique into decentralized stochastic optimization, resulting in the CHOCO-SGD algorithm, which supports biased compressors. For strongly convex objectives, the authors therein establish the method's convergence to the optimal solution at a rate of $\mathcal{O}\left(\frac{1}{nK}\right)$ under diminishing step sizes. Furthermore, they analyze the performance of CHOCO-SGD under general non-convex objectives~\cite{koloskova2020decentralized} and obtain a sublinear convergence rate of $\mathcal{O}\left(\frac{1}{\sqrt{nK}}\right)$, which matches that of the decentralized learning algorithm with exact communication~\cite{lian2017can}. In addition, Singh \textit{et al.}~\cite{singh2021squarm} develop a communication-efficient decentralized momentum method by combining communication compression with momentum acceleration strategy, and provide the convergence guarantees. Zhao \textit{et al.}~\cite{zhao2022beer} further apply communication compression to decentralized stochastic gradient tracking method, yielding the BEER algorithm, which enhances communication efficiency while mitigating the impact of data heterogeneity on convergence performance.

\subsection{Differentially private decentralized learning}
Differential privacy was first proposed in~\cite{dwork2006our} to protect data privacy for database queries, by adding randomly generated zero-mean noises to the output of a query function before it is exposed. Given the remarkable success of machine learning, there has been a recent surge in research efforts towards achieving DP guarantees in machine learning systems. DP guarantee is initially integrated into a centralized (single-node) setting for designing differentially private stochastic learning algorithms~\cite{abadi2016deep,wang2017differentially,iyengar2019towards,chen2020understanding,wang2020differentially}, and a baseline utility bound of $\mathcal{O}\left( \sqrt{d\log \left( \frac{1}{\delta} \right)}/\left( J\epsilon \right ) \right)$ for general non-convex problems is established~\cite{abadi2016deep}. Further, DP guarantee is considered in distributed learning with a server-client structure~\cite{mcmahan2017learning,li2019asynchronous,wang2019efficient,wu2020value,wei2020federated,zeng2021differentially,wei2021user,li2022soteriafl,liu2022loss,zhou2023optimizing,wei2023securing}, and a tight utility bound of $\mathcal{O}\left( \sqrt{d\log \left( \frac{1}{\delta} \right)}/(\sqrt{n}J\epsilon) \right)$ for general non-convex problems is provided~\cite{lowy2023private,zhou2023optimizing}, which scales as $1/\sqrt{n}$ w.r.t. the number of nodes $n$.

Recently, there have been a few works aiming to achieve differential privacy for fully decentralized learning algorithms. For example, the works in~\cite{cheng2018leasgd,cheng2019towards} achieve differential privacy in fully decentralized learning systems for strongly convex problems. Wang \textit{et al.}~\cite{wang2022tailoring} achieve differential privacy in fully decentralized architectures by tailoring gradient methods for deterministic optimization problems. Yu \textit{et al.}~\cite{yu2021decentralized} present a decentralized stochastic learning method for non-convex problems with DP guarantee (DP$^2$SGD) based on D-PSGD~\cite{lian2017can}, while providing no theoretical utility guarantee under a given privacy budget. Xu \textit{et al.}~\cite{xu2021dp} propose a differentially private asynchronous decentralized learning algorithm (A(DP)$^{2}$SGD) for non-convex problems based on AD-PSGD~\cite{lian2018asynchronous}, which provides privacy guarantee in the sense of R\'enyi differential privacy (RDP)~\cite{mironov2017renyi}. However, the utility bound established therein~\cite{xu2021dp} can not match that of the server-client distributed counterparts, losing a scaling factor of $1/\sqrt{n}$. It should be noted that the above-mentioned two fully decentralized differentially private algorithms~\cite{yu2021decentralized,xu2021dp} work only for undirected communication graphs, which is often not satisfied in practical scenarios. More Recently, the authors in~\cite{zhu2024privsgp} develop a differentially private decentralized algorithm over directed graphs, and establish a tight utility bound of $\mathcal{O}\left( \sqrt{d\log \left( \frac{1}{\delta} \right)}/(\sqrt{n}J\epsilon) \right)$ for their method.

It is worth noting that all of the above mentioned differentially private decentralized learning algorithms employ exact communication, resulting in low communication efficiency. There are several recent works that have explored differentially private distributed learning with communication compression, such as~\cite{agarwal2018cpsgd, li2022soteriafl, bao2023communication}, which are, however, limited to server-client architectures. To this end, we aim to develop a differentially private and communication-compressed learning algorithm for fully decentralized settings, while, more importantly, preserving strong model utility guarantee.

\section{Algorithm Development}
We consider solving Problem~\eqref{global_loss_function} over the following general network model.

\textbf{Network Model.}
The communication topology considered in this work is modeled as a directed graph $\mathcal{G}=\left( \mathcal{V},\mathcal{E} \right)$, where $\mathcal{V}=\{1,2,...,n\}$ denotes the set of nodes and $\mathcal{E} \subset \mathcal{V} \times \mathcal{V}$ denotes the set of directed edges/links of the graph. 
We associate the graph $\mathcal{G}$ with a non-negative mixing matrix $A=\left[ a_{ij} \right]  \in \mathbb{R}^{n \times n}$ such that $(j,i) \in \mathcal{E}$ if $a_{i,j} > 0$, i.e., there is a link from node $j$ to node $i$.
Without loss of generality, we assume that each node is an in-neighbor of itself. The sets of in-neighbors and out-neighbors of node $i$ are defined as:
\begin{equation*}
\begin{aligned}
\mathcal{N}_{i}^{in}:=\left\{ j\left| \left( j,i \right) \in \mathcal{E} \right. \right\} \cup \left\{ i \right\} ,
\\
\mathcal{N}_{i}^{out}:=\left\{ j\left| \left( i,j \right) \in \mathcal{E} \right. \right\} \cup \left\{ i \right\} .
\end{aligned}
\end{equation*}

The following assumptions are made on the mixing matrix and graph for the above network model to facilitate the subsequent utility analysis for the proposed algorithm.

\begin{Ass}[Graph Connectivity]
\label{assumption_mixing_matrix}
We assume that the graph $\mathcal{G}$ is strongly connected.
\end{Ass}

\begin{Ass}[Mixing Matrix]
\label{Ass_weight_matrix}
The non-negative mixing matrix $A$ is column-stochastic, i.e., $\mathbf{1}^\top A=\mathbf{1}^\top$, where $\mathbf{1}$ is a column vector with all of its elements equal to $1$.
\end{Ass}

Given the above assumptions, we next state a key result
from~\cite{nedic2014distributed} which will be useful in our analysis.

\begin{Pro}
\label{Pro}
Suppose Assumptions~\ref{assumption_mixing_matrix} and~\ref{Ass_weight_matrix} hold. Then, there exist a stochastic vector $\phi \in \mathbb{R}^n$, and constants $0 < \lambda < 1$ and $C > 0$ such that, for all $k \geqslant 0$, we have
\begin{equation}
\label{network_convergence}
\left\| A^k-\phi \mathbf{1}^\top \right\| \leqslant C\lambda ^k.
\end{equation}
Moreover, there exists a constant $\delta > 0$ such that, for all $i \in \{1,2,...,n\}$ and $k \geqslant 1$, we have
\begin{equation}
\left[ A^k \mathbf{1} \right] _i \geqslant \beta .
\end{equation}
\end{Pro}

Note that the column-stochastic property of the mixing matrix is considerably weaker than double-stochastic property. Each computing node $i$ can use its own out-degree to form the $i$’th column of mixing matrix. Thus the weight matrix can be constructed in the decentralized setting without each node knowing $n$ or the structure of the graph.

Before developing our proposed algorithm, we briefly introduce the following definition of $(\epsilon, \delta)$-DP~\cite{dwork2014algorithmic}, which is crucial to subsequent analysis.
\begin{Def}[$(\epsilon,\delta)$-DP]
A randomized mechanism $\mathcal{M}$ with domain $\mathcal{D}$ and range $\mathcal{R}$ satisfies $(\epsilon,\delta)$-differential privacy, or $(\epsilon,\delta)$-DP for short, if for any two adjacent inputs $\mathrm{x},\mathrm{x}^{\prime}\in \mathcal{D}$ differing on a single entry and for any subset of outputs $S\subseteq \mathcal{R}$, it holds that
\begin{equation}
Pr\left[ \mathcal{M}\left( \mathrm{x} \right) \in S \right] \leqslant e^{\epsilon}Pr\left[ \mathcal{M}\left( \mathrm{x}^{\prime} \right) \in S \right] +\delta ,
\end{equation}
where the privacy budget $\epsilon$ denotes the privacy lower bound to measure a randomized query and $\delta$ is the probability of
breaking this bound. The smaller the values of $\epsilon$ and $\delta$ are, the higher the level of privacy guarantee will be.
\end{Def}

\begin{algorithm}[t]
\caption{ {\alg} }
\label{DP-QSGP}
\textbf{Initialization}: $x_{i}^{1}=\hat{x}_{i}^{1}=\mathbf{0}^d$, $y_i^1=1$ and privacy budget $(\epsilon,\delta)$ for all $i \in \{1,2,...,n\}$, step size $\eta > 0$, and  total number of iterations $T$.
\begin{algorithmic}[1] 
\For{$t=1,2,3,...,T$, at node $i$,}
\State Compression:  $q_{i}^{t}=Q\left( x_{i}^{t}-\hat{x}_{i}^{t} \right) $ ;
\State Sends $\left\{q_{i}^{t}, y_{i}^{t} \right\}$ to all out-neighbors $k \in \mathcal{N}_i^{out}$;
\State Receives $\left\{q_{j}^{t}, y_{j}^{t} \right\}$ from all in-neighbors $j \in \mathcal{N}_i^{in}$;
\State Updates $\hat{x}_j^{t+1}$ for all in-neighbors $j \in \mathcal{N}_i^{in}$: 
\begin{equation*}
\hat{x}_j^{t+1} = \hat{x}_j^{t} + q_j^t;
\end{equation*}
\State Generates intermediate model parameter $w_i^{t+1}$ by: 
\begin{equation*}
w_i^{t+1}=x_{i}^{t}-\hat{x}_{i}^{t+1}+\sum_{j\in \mathcal{N}_{i}^{in}}{a_{ij}\hat{x}_{j}^{t+1}};
\end{equation*}
\State Updates $y_i^{t+1}$ by: $y_{i}^{t+1}=\sum_{j\in \mathcal{N}_{i}^{in}}{a_{ij}y_{j}^{t}}$;
\State Updates $z_i^{t+1}$ by:  
$z_{i}^{t+1}=\frac{w_i^{t+1}}{y_{i}^{t+1}}$
\State Randomly samples a local training data $\xi_i^{t+1}$ with the sampling probability $\frac{1}{J}$;
\State Computes stochastic gradient at $z_i^{t+1}$: $\nabla f_i(z_i^{t+1};\xi_i^{t+1})$
\State Draws randomized noise $N_i^{t+1}$ from the Gaussian distribution:  $N_{i}^{t+1}\sim \mathcal{N}\left( 0,\sigma^2\mathbb{I}_d \right)$;
\State Differentially private local SGD to updates $x_i^{t+1}$: 
\begin{equation*}
x_{i}^{t+1}=w_{i}^{t+1}-\eta \cdot \left( \nabla f_i\left( z_{i}^{t+1};\xi _{i}^{t+1} \right) +N_i^{t+1} \right) .
\end{equation*}
\EndFor
\State \textbf{end}
\end{algorithmic}
\end{algorithm}

Now, we present our differentially private decentralized learning algorithm with a communication compression strategy, termed {\alg}, which works over the aforementioned network model. The complete pseudo-code is summarized in Algorithm~\ref{DP-QSGP}.
At a high level, {\alg} is comprised of local SGD and the averaging of neighboring information, following a framework similar to SGP~\cite{assran2019stochastic} which employs the Push-Sum protocol~\cite{kempe2003gossip} to tackle the unblanceness of directed graphs. However, the key distinction lies in \textbf{i)} the employment of error feedback-based communication compression operation (c.f., line $2$-$5$ in Algorithm~\ref{DP-QSGP}) and \textbf{ii)} the injection of DP Gaussian noise before performing local SGD (c.f., line $11$-$12$ in Algorithm~\ref{DP-QSGP}). In particular, each node $i$ maintains five class of variables during the learning process: \textbf{a)} the model parameter $x_i^t$; \textbf{b)} the scalar Push-Sum weight $y_i^t$; \textbf{c)} the de-biased parameter $z_i^t=x_i^t/y_i^t$; \textbf{d)} the auxiliary parameter $\hat{x}_i^t$; and \textbf{e)} the estimates of the value of $x_j^t$ of its in-neighbors $j \in \mathcal{N}_{i}^{in}$, denoted by $\hat{x}_j^t$.

For variables, stochastic gradients and full gradients, we concatenate the row vectors corresponding to each node to form the following matrices:
\begin{equation*}
\begin{aligned}
& X^t:=\left[ x_{1}^{t};x_{2}^{t}\cdot \cdot \cdot x_{n}^{t} \right] \in \mathbb{R}^{n\times d}
\\
& \hat{X}^t:=\left[ \hat{x}_{1}^{t};\hat{x}_{2}^{t}\cdot \cdot \cdot \hat{x}_{n}^{t} \right] \in \mathbb{R}^{n\times d}
\\
& \partial F\left( Z^t;\xi ^t \right) :=\left[ \nabla f_1\left( z_{1}^{t};\xi _{1}^{t} \right) \cdot \cdot \cdot \nabla f_n\left( z_{n}^{t};\xi _{n}^{t} \right) \right] \in \mathbb{R}^{n\times d}
\\
& \partial f\left( Z^t \right) :=\left[ \nabla f_1\left( z_{1}^{t} \right) ;\nabla f_2\left( z_{2}^{t} \right) \cdot \cdot \cdot \nabla f_n\left( z_{n}^{t} \right) \right] \in \mathbb{R}^{n\times d}
\\
& N^t:=\left[ N_{1}^{t};N_{2}^{t};\cdot \cdot \cdot ,N_{n}^{t} \right] \in \mathbb{R}^{n\times d} .
\end{aligned}
\end{equation*}

To this end, the system updates at each iteration $k$ can be written in matrix notation as 
\begin{subequations}
\label{concatenate_ierate}
\begin{align}
& Q^t =Q\left( X^t-\hat{X}^t \right) 
\label{update_1},
\\
& \hat{X}^{t+1} =\hat{X}^t+Q^t
\label{update_2},
\\
& W^{t+1} =X^t+\left( A-I \right) \hat{X}^{t+1}
\label{update_3},
\\
& y^{t+1}  =Ay^t
\label{update_4},
\\
&z_{i}^{t+1}  =\frac{w_{i}^{t+1}}{y_{i}^{t+1}}
\label{update_5},
\\
& X^{t+1} =W^{t+1}-\eta \cdot \left( \partial F\left( Z^{t+1};\xi ^{t+1} \right) +N^{t+1} \right) 
\label{update_6},
\end{align}
\end{subequations}
where $\eta > 0$ is the step size. $N_i^t$ denotes the injected random noise to ensure DP guarantee for node $i$ at iteration $t$, which is drawn from a Gaussian distribution with variance $\sigma^2$. Without loss of generality, we assume that the parameters $x_i$ and $\hat{x}_i$ of each node $i$ are initialized with zero vectors, to simplify the subsequent model utility analysis. 

\begin{Ass}[Initialization]
\label{assumption_initialization}
The parameters $x_i$ and $\hat{x}_i$ are initialized with $\mathbf{0} \in \mathbb{R}^d$ and $y_i^1 = 1$ for all $i \in \{ 1,2,...,n\}$.
\end{Ass}

In addition, we make the following assumption on our used communication compression operator.

\begin{Ass}[Compression Operator]
\label{quantization_operator}
The compression operator $Q: \mathbb{R}^d \longrightarrow \mathbb{R}^d$ satisfies for all $x \in \mathbb{R}^d$:
\begin{equation}
\label{quantization_ratio}
\mathbb{E}\left[ \left\| Q\left( x \right) -x \right\| ^2 \right] \leqslant \omega ^2\left\| x \right\| ^2,
\end{equation}
where $0 \leqslant \omega < 1$ is the non-negative compression coefficient.
\end{Ass}

\begin{Rem}
Compared with the unbiased compression operator used in~\cite{mishchenko2024distributed} and~\cite{alistarh2017qsgd}, the compression operator in~\eqref{quantization_ratio} is not required to be unbiased, i.e., $\mathbb{E}\left[ Q\left( x \right) \right] =x$, and therefore supports a larger class of compression operators. Example operators that satisfy Assumption~\ref{quantization_operator} are detailed in Section~\ref{section_experiment}
\end{Rem}

\section{Theoretical Analysis}
In this section, we provide the privacy and utility guarantees for our proposed {\alg}(Algorithm~\ref{DP-QSGP}). To this end, we denote $\bar{x}^t\triangleq \frac{1}{n}\sum_{i=1}^n{x_{i}^{t}}$ as the average of $x_i^t$ for all nodes, $\left\| \cdot \right\|_2$ as matrix spectral norm, and $\left\| \cdot \right\|$ as Frobenius norm.

We first make the following assumptions on the local objective functions and stochastic gradients of each node, to facilitate the subsequent analysis.

\begin{Ass}[Smoothness]
\label{assumption_smooth_saga}
For each local function $f_i,i\in\mathcal{V}$, there exists a constant $L>0$ such that 
\begin{equation}
\left\| \nabla f_i\left( x \right) -\nabla f_i\left( y \right) \right\| \leqslant L\left\| x-y \right\| , \forall x,y\in \mathbb{R}^d .
\end{equation}
\end{Ass}

\begin{Ass}[Unbiased Gradient]
\label{assumption_unbiased_gradient}
For any model $x\in \mathbb{R}^d$, the stochastic gradient $\nabla f_i\left( x;\xi _i \right), \xi_i \sim  \{1,2,...,J\}$ generated by each node $i$ is unbiased, i.e.,
\begin{equation}
\mathbb{E}\left[ \nabla f_i\left( x;\xi _i \right) \right] =\nabla f_i\left( x \right) .
\end{equation}
\end{Ass}

\begin{Ass}[Bounded Gradient]
\label{assumption_bounded_gradient}
There exists a finite positive constant $b$ such that each per-sample gradient $\nabla f_i\left( x;\xi _i \right)$ is upper-bounded by $G$, i.e., for all $i \in \{1,2,...,n\}$
\begin{equation}
\left\| \nabla f_i\left( x;\xi _i \right) \right\| \leqslant G, \forall x\in \mathbb{R}^d .
\end{equation}
\end{Ass}

Assumption~\ref{assumption_smooth_saga} and~\ref{assumption_unbiased_gradient} are commonly used for the convergence analysis in stochastic optimization, and Assumption~\ref{assumption_bounded_gradient} is also standard for the differential privacy analysis~\cite{bassily2014private, iyengar2019towards, jayaraman2018distributed}.

The following proposition shows that DP guarantee for each node can be achieved by setting the the variance of Gaussian noise $\sigma^2$ properly according to the given certain privacy budget $(\epsilon,\delta)$ and the total number of iterations $T$.

\begin{Pro}[Privacy Guarantee]
\label{lemma_of_sigma}
There exist constants $c_1$ and $c_2$ such that, for any $\epsilon <\frac{c_1T}{J^2}$ and $\delta \in \left( 0,1 \right) $, $(\epsilon,\delta)$-DP can be guaranteed for each node $i$ for {\alg} after $T$ iterations if we set
\begin{equation}
\label{value_of_sigma}
\sigma ^2=\frac{Tc_{2}^{2}G^2\log \left( \frac{1}{\delta} \right)}{J^2\epsilon ^2}.
\end{equation}
\end{Pro}

\begin{proof}
The proof of the above result can be easily adapted from Theorem 1 in~\cite{abadi2016deep} by knowing the fact that the sampling probability is $\frac{1}{J}$ for each node $i$ at each iteration.
\end{proof}

With the above assumptions and proposition, we can further have the following utility result for {\alg} (Algorithm~\ref{DP-QSGP}).

\begin{Thm}[Utility Guarantee]
\label{Theorem_utility} 
Suppose Assumptions~\ref{assumption_mixing_matrix}-\ref{assumption_bounded_gradient} hold and $J\geqslant \frac{c_2\sqrt{d\log \left( \frac{1}{\delta} \right)}\cdot n^{\frac{5}{2}}}{\epsilon}$ for a given privacy budget $(\epsilon,\delta)$.
For all $\omega \leqslant \left[ 10\left( 1+\gamma ^2 \right) \left( 1+\frac{4C^2}{\left( 1-\lambda \right) ^2} \right) \right] ^{-\frac{1}{2}}$, if we set $\eta =1/\left( \frac{J\epsilon}{c_2\sqrt{nd\log \left( \frac{1}{\delta} \right)}}+L \right)$, $T=\frac{J^2\epsilon ^2}{c_{2}^{2}d\log \left( \frac{1}{\delta} \right)}$ and the noise variance $\sigma ^2=\frac{Tc_{2}^{2}G^2\log \left( \frac{1}{\delta} \right)}{J^2\epsilon ^2}$, {\alg} can achieve $(\epsilon,\delta)$-DP guarantee for each node and has the following utility bound
\begin{equation}
\label{main_utility_bound}
\frac{1}{T}\sum_{t=1}^T{\mathbb{E}\left[ \left\| \nabla f\left( \bar{x}^t \right) \right\| ^2 \right]}\leqslant \mathcal{O}\left( \frac{\sqrt{d\log \left( \frac{1}{\delta} \right)}}{\sqrt{n}J\epsilon} \right) ,
\end{equation}
where the definition of $C$ and $\lambda$ can be found in Proposition~\ref{Pro}, and $\gamma \triangleq \left\| A-I \right\| _2$ with $I$ the identity matrix.
\end{Thm}

\begin{proof}
Substituting~\eqref{update_3} into~\eqref{update_6}, we have
\begin{equation*}
X^{t+1}=X^t+\left( A-I \right) \hat{X}^{t+1}-\eta \cdot \left( \partial F\left( Z^{t+1};\xi ^{t+1} \right) +N^{t+1} \right) .
\end{equation*}
Left multiplying $\mathbf{1}^{\top}$ on the both sides, and using that $\mathbf{1}^{\top}A=\mathbf{1}^{\top}$ according to Assumption~\ref{Ass_weight_matrix}, we have
\begin{equation*}
\mathbf{1}^{\top}X^{t+1}=\,\,\mathbf{1}^{\top}X^t-\eta \cdot \,\,\mathbf{1}^{\top}\left( \partial F\left( Z^{t+1};\xi ^{t+1} \right) +N^{t+1} \right) .
\end{equation*}
Dividing by $n$ both sides and denoting $\bar{x}^t\triangleq \frac{\mathbf{1}^{\top}X^t}{n}=\frac{1}{n}\sum_{i=1}^n{x_{i}^{t}}$, we have the update for the average iterate:
\begin{equation}
\label{average_iterate}
\bar{x}^{t+1}=\bar{x}^t-\eta \cdot \frac{1}{n}\left( \sum_{i=1}^n{\nabla f_i\left( z_{i}^{t+1};\xi _{i}^{t+1} \right)}+\sum_{i=1}^n{N_{i}^{t+1}} \right) .
\end{equation}
Using the $L$-smoothness of the global objective function which is implied by Assumption~\ref{assumption_smooth_saga}, we have
\begin{equation}
\begin{aligned}
& f\left( \bar{x}^{t+1} \right) 
\\
\leqslant & f\left( \bar{x}^t \right) +\left< \nabla f\left( \bar{x}^t \right) ,\bar{x}^{t+1}-\bar{x}^t \right> +\frac{L}{2}\left\| \bar{x}^{t+1}-\bar{x}^t \right\| ^2
\\
=& f\left( \bar{x}^t \right) -\eta \left< \nabla f\left( \bar{x}^t \right) ,\frac{1}{n}\sum_{i=1}^n{\left( \nabla f_i\left( z_{i}^{t+1};\xi _{i}^{t+1} \right) +N_{i}^{t+1} \right)} \right> 
\\
& +\frac{\eta ^2L}{2}\left\| \frac{1}{n}\sum_{i=1}^n{\nabla f_i\left( z_{i}^{t+1};\xi _{i}^{t+1} \right)}+\frac{1}{n}\sum_{i=1}^n{N_{i}^{t+1}} \right\| ^2
\\
=& f\left( \bar{x}^t \right) -\eta \left< \nabla f\left( \bar{x}^t \right) ,\frac{1}{n}\sum_{i=1}^n{\left( \nabla f_i\left( z_{i}^{t+1};\xi _{i}^{t+1} \right) +N_{i}^{t+1} \right)} \right> 
\\
& +\frac{\eta ^2L}{2}\left\| \frac{1}{n}\sum_{i=1}^n{\nabla f_i\left( z_{i}^{t+1};\xi _{i}^{t+1} \right)} \right\| ^2+\frac{\eta ^2L}{2}\left\| \frac{1}{n}\sum_{i=1}^n{N_{i}^{t+1}} \right\| ^2
\\
& +\eta ^2L\left< \frac{1}{n}\sum_{i=1}^n{\nabla f_i\left( z_{i}^{t+1};\xi _{i}^{t+1} \right)},\frac{1}{n}\sum_{i=1}^n{N_{i}^{t+1}} \right> .
\end{aligned}
\end{equation}
Taking the expectation of both sides, we obtain
\begin{equation}
\label{descent_lemma}
\begin{aligned}
& \mathbb{E}\left[ f\left( \bar{x}^{t+1} \right) \right] 
\\
\leqslant & \mathbb{E}\left[ f\left( \bar{x}^t \right) \right] -\eta \mathbb{E}\left[ \left< \nabla f\left( \bar{x}^t \right) ,\frac{1}{n}\sum_{i=1}^n{\nabla f_i\left( z_{i}^{t+1} \right)} \right> \right] 
\\
& +\frac{\eta ^2L}{2}\mathbb{E}\left[ \left\| \frac{1}{n}\sum_{i=1}^n{\nabla f_i\left( z_{i}^{t+1};\xi _{i}^{t+1} \right)} \right\| ^2 \right] 
\\
& +\frac{\eta ^2L}{2}\mathbb{E}\left[ \left\| \frac{1}{n}\sum_{i=1}^n{N_{i}^{t+1}} \right\| ^2 \right] .
\end{aligned}
\end{equation}
For the second term in the RHS of~\eqref{descent_lemma}, using the fact $\left< a,b \right> =\frac{1}{2}\left\| a \right\| ^2+\frac{1}{2}\left\| b \right\| ^2-\frac{1}{2}\left\| a-b \right\| ^2$, we have
\begin{equation}
\label{eq_9}
\begin{aligned}
& \left< \nabla f\left( \bar{x}^t \right) ,\frac{1}{n}\sum_{i=1}^n{\nabla f_i\left( z_{i}^{t+1} \right)} \right> 
\\
= & \frac{1}{2}\left\| \nabla f\left( \bar{x}^t \right) \right\| ^2+\frac{1}{2}\left\| \frac{1}{n}\sum_{i=1}^n{\nabla f_i\left( z_{i}^{t+1} \right)} \right\| ^2
\\
& -\frac{1}{2}\left\| \frac{1}{n}\sum_{i=1}^n{\nabla f_i\left( z_{i}^{t+1} \right)}-\nabla f\left( \bar{x}^t \right) \right\| ^2 .
\end{aligned}
\end{equation}
The last term in the above inequality can be bounded by
\begin{equation}
\label{eq_10}
\begin{aligned}
& \left\| \frac{1}{n}\sum_{i=1}^n{\nabla f_i\left( z_{i}^{t+1} \right)}-\nabla f\left( \bar{x}^t \right) \right\| ^2
\\
= & \left\| \frac{1}{n}\sum_{i=1}^n{\left( \nabla f_i\left( z_{i}^{t+1} \right) -\nabla f_i\left( \bar{x}^t \right) \right)} \right\| ^2
\\
\leqslant & \frac{1}{n}\sum_{i=1}^n{\left\| \nabla f_i\left( z_{i}^{t+1} \right) -\nabla f_i\left( \bar{x}^t \right) \right\| ^2}
\\
\overset{\left( a \right)}{\leqslant} & \frac{L^2}{n}\sum_{i=1}^n{\left\| z_{i}^{t+1}-\bar{x}^t \right\| ^2},
\end{aligned}
\end{equation}
where in $(a)$ we used Assumption~\ref{assumption_smooth_saga}. Substituting~\eqref{eq_10} into~\eqref{eq_9}, we have
\begin{equation}
\label{second_term}
\begin{aligned}
& \left< \nabla f\left( \bar{x}^t \right) ,\frac{1}{n}\sum_{i=1}^n{\nabla f_i\left( z_{i}^{t+1} \right)} \right> 
\\
\geqslant & \frac{1}{2}\left\| \nabla f\left( \bar{x}^t \right) \right\| ^2+\frac{1}{2}\left\| \frac{1}{n}\sum_{i=1}^n{\nabla f_i\left( z_{i}^{t+1} \right)} \right\| ^2
\\
& -\frac{L^2}{2n}\sum_{i=1}^n{\left\| z_{i}^{t+1}-\bar{x}^t \right\| ^2}.
\end{aligned}
\end{equation}
For the third term in the RHS of~\eqref{descent_lemma}, we can bound it by
\begin{equation}
\label{third_term}
\begin{aligned}
& \mathbb{E}\left[ \left\| \frac{1}{n}\sum_{i=1}^n{\nabla f_i\left( z_{i}^{t+1};\xi _{i}^{t+1} \right)} \right\| ^2 \right] 
\\
= & \mathbb{E}\left[ \left\| \frac{1}{n}\sum_{i=1}^n{\left[ \nabla f_i\left( z_{i}^{t+1};\xi _{i}^{t+1} \right) -\nabla f_i\left( z_{i}^{t+1} \right) +\nabla f_i\left( z_{i}^{t+1} \right) \right]} \right\| ^2 \right] 
\\
= & \mathbb{E}\left[ \left\| \frac{1}{n}\sum_{i=1}^n{\left( \nabla f_i\left( z_{i}^{t+1};\xi _{i}^{t+1} \right) -\nabla f_i\left( z_{i}^{t+1} \right) \right)} \right\| ^2 \right] 
\\
& +2\mathbb{E}\left[ \left< \frac{1}{n}\sum_{i=1}^n{\left( \nabla f_i\left( z_{i}^{t+1};\xi _{i}^{t+1} \right) -\nabla f_i\left( z_{i}^{t+1} \right) \right)}, \right. \right. 
\\
& \left. \left. \frac{1}{n}\sum_{i=1}^n{\nabla f_i\left( z_{i}^{t+1} \right)} \right> \right] +\mathbb{E}\left[ \left\| \frac{1}{n}\sum_{i=1}^n{\nabla f_i\left( z_{i}^{t+1} \right)} \right\| ^2 \right] 
\\
= & \frac{1}{n^2}\sum_{i=1}^n{\mathbb{E}\left[ \left\| \nabla f_i\left( z_{i}^{t+1};\xi _{i}^{t+1} \right) -\nabla f_i\left( z_{i}^{t+1} \right) \right\| ^2 \right]}
\\
& +\mathbb{E}\left[ \left\| \frac{1}{n}\sum_{i=1}^n{\nabla f_i\left( z_{i}^{t+1} \right)} \right\| ^2 \right] 
\\
\overset{\left( a \right)}{\leqslant} & \frac{4G^2}{n}+\mathbb{E}\left[ \left\| \frac{1}{n}\sum_{i=1}^n{\nabla f_i\left( z_{i}^{t+1} \right)} \right\| ^2 \right] ,
\end{aligned}
\end{equation}
where in $(a)$ we used $\left\| \nabla f_i\left( z_{i}^{t+1};\xi _{i}^{t+1} \right) \right\| \leqslant G$ according to Assumption~\ref{assumption_bounded_gradient} and $\left\| \nabla f_i\left( z_{i}^{t+1} \right) \right\| \leqslant G$ according to the definition of $f_i \left( \cdot  \right)$.

For the last term in the RHS of~\eqref{descent_lemma}, since Gaussian noises $\{N_i^{t+1}\}_{i=1,2,...,n}$ are independent with each other, we have 
\begin{equation}
\label{forth_term}
\mathbb{E}\left[ \left\| \frac{1}{n}\sum_{i=1}^n{N_{i}^{t+1}} \right\| ^2 \right] =\frac{1}{n^2}\sum_{i=1}^n{\mathbb{E}\left[ \left\| N_{i}^{t+1} \right\| ^2 \right]}=\frac{d\sigma ^2}{n}.
\end{equation}
Substituting~\eqref{second_term},~\eqref{third_term} and~\eqref{forth_term} into~\eqref{descent_lemma}, we have
\begin{equation}
\begin{aligned}
& \mathbb{E}\left[ f\left( \bar{x}^{t+1} \right) \right] 
\\
\leqslant & \mathbb{E}\left[ f\left( \bar{x}^t \right) \right] -\frac{\eta}{2}\mathbb{E}\left[ \left\| \nabla f\left( \bar{x}^t \right) \right\| ^2 \right] 
\\
& -\frac{\eta -\eta ^2L}{2}\mathbb{E}\left[ \left\| \frac{1}{n}\sum_{i=1}^n{\nabla f_i\left( z_{i}^{t+1} \right)} \right\| ^2 \right] 
\\
& +\frac{\eta L^2}{2n}\sum_{i=1}^n{\mathbb{E}\left[ \left\| z_{i}^{t+1}-\bar{x}^t \right\| ^2 \right]}+\frac{\eta ^2L}{2n}\left( 4G^2+d\sigma ^2 \right) .
\end{aligned}
\end{equation}
Substituting the consensus error bound $\mathbb{E}\left[ \left\| z_{i}^{t+1}- \bar{x}^t \right\| ^2 \right]$ (c.f.,~\eqref{consensus_error_bound} in the appendix) into the above inequality, we have
\begin{equation}
\label{eq_11}
\begin{aligned}
\mathbb{E}\left[ f\left( \bar{x}^{t+1} \right) \right] \leqslant & \mathbb{E}\left[ f\left( \bar{x}^t \right) \right] -\frac{\eta}{2}\mathbb{E}\left[ \left\| \nabla f\left( \bar{x}^t \right) \right\| ^2 \right] 
\\
& -\frac{\eta -\eta ^2L}{2}\mathbb{E}\left[ \left\| \frac{1}{n}\sum_{i=1}^n{\nabla f_i\left( z_{i}^{t+1} \right)} \right\| ^2 \right] 
\\
& +\frac{5\eta ^3L^2C^2}{\beta ^2\left( 1-\lambda \right) ^2}\left[ 2n\left( G^2+d\sigma ^2 \right) +\zeta \right] 
\\
& +\frac{\eta ^2L}{2n}\left( 4G^2+d\sigma ^2 \right) .
\end{aligned}
\end{equation}
Summing~\eqref{eq_11} from $t=0$ to $T$, we have
\begin{equation}
\label{eq_12}
\begin{aligned}
& \frac{\eta}{2}\sum_{t=1}^T{\mathbb{E}\left[ \left\| \nabla f\left( \bar{x}^t \right) \right\| ^2 \right]}
\\
& +\frac{\eta \left( 1-\eta L \right)}{2}\sum_{t=1}^T{\mathbb{E}\left[ \left\| \frac{1}{n}\sum_{i=1}^n{\nabla f_i\left( z_{i}^{t+1} \right)} \right\| ^2 \right]}
\\
\leqslant & f\left( \bar{x}^1 \right) -f^*+\frac{5\eta ^3L^2C^2}{\beta ^2\left( 1-\lambda \right) ^2}\left[ 2n\left( G^2+d\sigma ^2 \right) +\zeta \right] T
\\
& +\frac{\eta ^2L}{2n}\left( 4G^2+d\sigma ^2 \right) T .
\end{aligned}
\end{equation}
Multiplying $\frac{2}{\eta T}$ on both sides of~\eqref{eq_12}, we obtain
\begin{equation}
\label{eeq_1}
\begin{aligned}
& \frac{1}{T}\sum_{t=1}^T{\mathbb{E}\left[ \left\| \nabla f\left( \bar{x}^t \right) \right\| ^2 \right]}
\\
& +\frac{1-\eta L}{T}\sum_{t=1}^T{\mathbb{E}\left[ \left\| \frac{1}{n}\sum_{i=1}^n{\nabla f_i\left( z_{i}^{t+1} \right)} \right\| ^2 \right]}
\\
\leqslant & \frac{2\left( f\left( \bar{x}^1 \right) -f^* \right)}{\eta T}+\frac{\eta L}{n}\left( 4G^2+d\sigma ^2 \right) 
\\
& +\frac{10\eta ^2L^2C^2}{\beta ^2\left( 1-\lambda \right) ^2}\left[ 2n\left( G^2+d\sigma ^2 \right) +\zeta \right] .
\end{aligned}
\end{equation}
According to~\eqref{rho_upper_bound} and~\eqref{def_zeta}, we know that
\begin{equation}
\label{zeta_upper_bound}
\zeta \leqslant n\left( G^2+d\sigma ^2 \right) .
\end{equation}
Invoking~\ref{zeta_upper_bound} into~\eqref{eeq_1}, we have
\begin{equation}
\label{eeq_2}
\begin{aligned}
& \frac{1}{T}\sum_{t=1}^T{\mathbb{E}\left[ \left\| \nabla f\left( \bar{x}^t \right) \right\| ^2 \right]}
\\
& +\frac{1-\eta L}{T}\sum_{t=1}^T{\mathbb{E}\left[ \left\| \frac{1}{n}\sum_{i=1}^n{\nabla f_i\left( z_{i}^{t+1} \right)} \right\| ^2 \right]}
\\
\leqslant & \frac{2\left( f\left( \bar{x}^1 \right) -f^* \right)}{\eta T}+\frac{\eta L}{n}\left( 4G^2+d\sigma ^2 \right) 
\\
& +\frac{30n\eta ^2L^2C^2}{\beta ^2\left( 1-\lambda \right) ^2}\left( G^2+d\sigma ^2 \right) .
\end{aligned}
\end{equation}
Letting the step size $\eta =\frac{1}{\sqrt{\frac{T}{n}}+L}$, \eqref{eeq_2} can be relaxed as
\begin{equation}
\label{eeq_3}
\begin{aligned}
& \frac{1}{T}\sum_{t=1}^T{\mathbb{E}\left[ \left\| \nabla f\left( \bar{x}^t \right) \right\| ^2 \right]}
\\
\overset{1-\eta L\geqslant 0}{\leqslant} & \frac{1}{T}\sum_{t=1}^T{\mathbb{E}\left[ \left\| \nabla f\left( \bar{x}^t \right) \right\| ^2 \right]}
\\
& +\frac{1-\eta L}{T}\sum_{t=1}^T{\mathbb{E}\left[ \left\| \frac{1}{n}\sum_{i=1}^n{\nabla f_i\left( z_{i}^{t+1} \right)} \right\| ^2 \right]}
\\
\leqslant & \frac{2\left( f\left( \bar{x}^1 \right) -f^* \right)}{\sqrt{nT}}+\frac{2L\left( f\left( \bar{x}^1 \right) -f^* \right)}{T}
\\
& +\frac{L\left( 4G^2+d\sigma ^2 \right)}{\sqrt{nT}} 
\\
& +\frac{1}{T}\cdot \frac{30n^2L^2C^2}{\beta ^2\left( 1-\lambda \right) ^2}\left( G^2+d\sigma ^2 \right). 
\end{aligned}
\end{equation}
Substituting $\sigma^2$ in~\eqref{value_of_sigma} into~\eqref{eeq_3}, and setting $T=\frac{J^2\epsilon ^2}{c_{2}^{2}d\log \left( \frac{1}{\delta} \right)}$, we can further obtain
\begin{equation}
\label{eeq_5}
\begin{aligned}
& \frac{1}{T}\sum_{t=1}^T{\mathbb{E}\left[ \left\| \nabla f\left( \bar{x}^t \right) \right\| ^2 \right]}
\\
\leqslant & \frac{c_2\sqrt{d\log \left( \frac{1}{\delta} \right)}}{\sqrt{n}J\epsilon}\cdot \left[ 2\left( f\left( \bar{x}^1 \right) -f^* \right) +5LG^2 \right] 
\\
& +\frac{c_{2}^{2}d\log \left( \frac{1}{\delta} \right)}{J^2\epsilon ^2}\cdot \left[ 2L\left( f\left( \bar{x}^1 \right) -f^* \right) \right] 
\\
& +\frac{c_{2}^{2}d\log \left( \frac{1}{\delta} \right)}{J^2\epsilon ^2}\cdot \frac{60n^2L^2C^2G^2}{\beta ^2\left( 1-\lambda \right) ^2}.
\end{aligned}
\end{equation}
Under mild assumption of
\begin{equation}
\label{mild_assumption}
J\geqslant \frac{c_2\sqrt{d\log \left( \frac{1}{\delta} \right)}\cdot n^{\frac{5}{2}}}{\epsilon},
\end{equation}
\eqref{eeq_5} can be reduced as
\begin{equation}
\begin{aligned}
& \frac{1}{T}\sum_{t=1}^T{\mathbb{E}\left[ \left\| \nabla f\left( \bar{x}^t \right) \right\| ^2 \right]}
\\
\leqslant & \frac{c_2\sqrt{d\log \left( \frac{1}{\delta} \right)}}{\sqrt{n}J\epsilon}\cdot \left[ 2\left( f\left( \bar{x}^1 \right) -f^* \right) +5LG^2 \right] 
\\
& +\frac{c_2\sqrt{d\log \left( \frac{1}{\delta} \right)}}{\sqrt{n}J\epsilon}\cdot \left[ 2L\left( f\left( \bar{x}^1 \right) -f^* \right) \right] 
\\
& +\frac{c_2\sqrt{d\log \left( \frac{1}{\delta} \right)}}{\sqrt{n}J\epsilon}\cdot \frac{60L^2C^2G^2}{\beta ^2\left( 1-\lambda \right) ^2}
\\
= & \mathcal{O}\left( \frac{\sqrt{d\log \left( \frac{1}{\delta} \right)}}{\sqrt{n}J\epsilon} \right) ,
\end{aligned}
\end{equation}
which completes the proof.
\end{proof}

\begin{Rem}
{\alg} achieves the same utility guarantee as differentially private learning algorithms with a server-client structure, such as LDP SVRG/SPIDER~\cite{lowy2023private}. Furthermore, {\alg} recovers the baseline utility $\mathcal{O}\left( \sqrt{d\log \left( \frac{1}{\delta} \right)}/J\epsilon \right) $ of the centralized DP-SGD~\cite{abadi2016deep} in the single node case with $n=1$. In addition, the derived utility bound matches that of differentially private decentralized learning methods with exact communication~\cite{zhu2024privsgp}, demonstrating that our {\alg} maintains strong model utility while ensuring both DP guarantee and communication efficiency.
\end{Rem}

\section{Experiments}
\label{section_experiment}
In this section, we conduct extensive experiments to evaluate the performance of the proposed DP-CSGP against differentially private decentralized algorithm DP$^2$SGD with exact communication~\cite{yu2021decentralized}. All experiments are deployed in a high performance computer with Intel Xeon E5-2680 v4 CPU @ 2.40GHz and 8 Nvidia RTX 4090 GPUs, and are implemented with distributed communication package \textit{torch.distributed} in PyTorch, where a process serves as a node, and inter-process communication is used to mimic communication among nodes.

\subsection{Experimental Setup}
We consider two benchmark non-convex learning tasks (i.e., training deep CNN ResNet-18~\cite{he2016deep} on Cifar-10 dataset~\cite{krizhevsky2009learning}, and training shallow 2-layer neural network on Mnist dataset~\cite{deng2012mnist} dataset), in fully decentralized setting composed of 10 nodes. For all experiments, we evenly split the shuffled datasets across 10 nodes, and employ a directed exponential graph as the communication topology. The learning rate is set to be $0.03$ for ResNet-18 training task and $0.01$ for shallow 2-layer neural network training. Privacy parameters $\delta$ is set to be $10^{-4}$, and we test different values for $\epsilon$ which implies different levels of privacy guarantee. It is worthy to point out that, in order to achieve privacy guarantee, bounded gradient (Assumption~\ref{assumption_bounded_gradient}) is required. However, it is not easy to obtain this upper bound $G$ or it is somewhat large especially for neural networks. Therefore, following experiments in previous works~\cite{zhang2020private, lowy2023private, ding2021differentially}, we also apply gradient clipping (i.e. $\mathrm{Clip}_G\left( \nabla f_i\left( z_{i}^{t};\xi _{i}^{t} \right) \right) =\nabla f_i\left( z_{i}^{t};\xi _{i}^{t} \right) \cdot \min \left( 1,G/\left\| \nabla f_i\left( z_{i}^{t};\xi _{i}^{t} \right) \right\| \right) $) in our experiments. In particular, we choose $G=1.5$ for ResNet-18 training and $G=0.5$ for 2-layer neural network training. For the privacy noise $N$, we will set the variances of $N$ according to their theoretical values implied in Proposition~\ref{lemma_of_sigma}. Note that all experimental results are averaged over five repeated runs.

\textbf{Compression Schemes.}
We implement two compression schemes that satisfy Assumption~\ref{quantization_operator} for our proposed {\alg}: $\mathrm{rand_a}$-sparsification and $\mathrm{gsgd_b}$-quantization. The exact definitions of these two schemes are given below.
\begin{itemize}
    \item $\mathrm{rand_a}$-sparsification~\cite{wangni2018gradient}. The $\mathrm{rand_a}: \mathbb{R}^d\rightarrow \mathbb{R}^d$ compression operator (for $0<a<1$) preserves a randomly chosen $a$ fraction of the values of the vector and sets the other ones to zero. Only $32 \lfloor ad \rfloor$ bits are required to send $\mathrm{rand}_a\left( x \right)$ to another node — all the values of all the values of non-zero entries (note that entries are represented as 32 $\mathrm{float} 32$ numbers in $\mathrm{torch.Tensor}$). Receiver can recover positions of these entries if it knows the random seed of uniform sampling operator used to select these entries. This random seed could be communicated once before running the algorithm. This sparsification scheme satisfies Assumption~\ref{quantization_operator} with coefficient $\omega^2 =1-a$. 
    
    \item $\mathrm{gsgd_b}$-quantification~\cite{alistarh2017qsgd}. The $\mathrm{gsgd_b}: \mathbb{R}^d\rightarrow \mathbb{R}^d$ compression operator (for $b>1$) is given as $\mathrm{gsgd}_b\left( x \right) :=\left\| x \right\| \cdot \mathrm{sig}\left( x \right) \cdot 2^{-\left( b-1 \right)}\cdot \lfloor 2^{b-1}\left| x \right|/\left\| x \right\| +u \rfloor $, where $u\thicksim _{u.a.r.}\left[ 0,1 \right] ^d$ is a random dithering vector and $\mathrm{sig}\left( x \right)$ assigns the element-wise sign: $\left[ \mathrm{sig}\left( x \right) \right] _i=1$ if $\left[ x \right] _i\geqslant 0$ and $\left[ \mathrm{sig}\left( x \right) \right] _i=-1$ if $\left[ x \right] _i<0$. As the value in the right bracket will be rounded to an integer in $\left\{ 0,,...,2^{b-1} \right\}$, each coordinate can be encoded with at most $\left( b-1 \right) +1$ bits ($1$ for the sign). This quantization scheme satisfies Assumption~\ref{quantization_operator} with coefficient $\omega ^2=\min \left\{ d/2^{2\left( b-1 \right)},\sqrt{d}/2^{b-1} \right\}$. 
\end{itemize}

\subsection{Experimental Results}
For the 2-layer neural network training task, we consider three levels of privacy with $\epsilon=0.2, 0.3, 0.5$ and a common $\delta=10^{-4}$. For our {\alg}, we test different levels of sparsification operator ($\mathrm{rand}_{50}$, $\mathrm{rand}_{75}$, $\mathrm{rand}_{10}$) and different levels of quantification operator ($\mathrm{gsgd}_{16}$, $\mathrm{gsgd}_{8}$), where the experimental results are reported in Fig.~\ref{Mnist_sparsification} and~\ref{Mnist_quantification} respectively. The experimental results shown in Fig.~\ref{mnist_sparsify_acc_epsilon_0_point_5}~\ref{mnist_sparsify_acc_epsilon_0_point_3}~\ref{mnist_sparsify_acc_epsilon_0_point_2} and Fig.~\ref{mnist_quantify_acc_epsilon_0_point_5}~\ref{mnist_quantify_acc_epsilon_0_point_3}~\ref{mnist_quantify_acc_epsilon_0_point_2} illustrate that, under the same total privacy budget, our {\alg} converges faster than the uncompressed decentralized private counterpart DP$^2$SGD in terms of communication bits, while ultimately achieving comparable model accuracy with less communication resource consumption. This validates that communication compression indeed provide significant savings in terms of communication cost. In addition, it can be observed from Fig.~\ref{mnist_sparsify_acc_epsilons} and~\ref{Mnist_quantification} that, when setting the same total iteration and employing the same level of compression operator, the model performance of {\alg} decreases as the privacy budget $\epsilon$ decreases (implying the level of privacy guarantee being stronger), indicating the inherent trade-offs between model utility and privacy guarantee.

\begin{figure*}[t]
\centering
\subfloat[$(0.5, 10^{-4})$-DP]{
\includegraphics[width=0.24\linewidth]{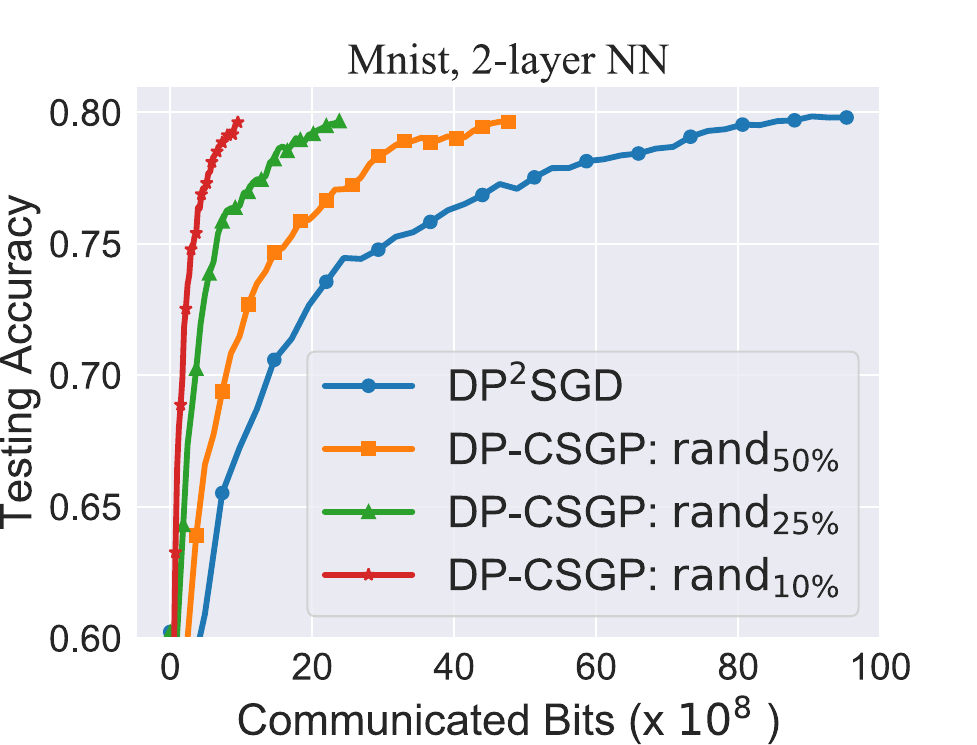}
\label{mnist_sparsify_acc_epsilon_0_point_5}
}
\subfloat[$(0.3, 10^{-4})$-DP]{
\includegraphics[width=0.24\linewidth]{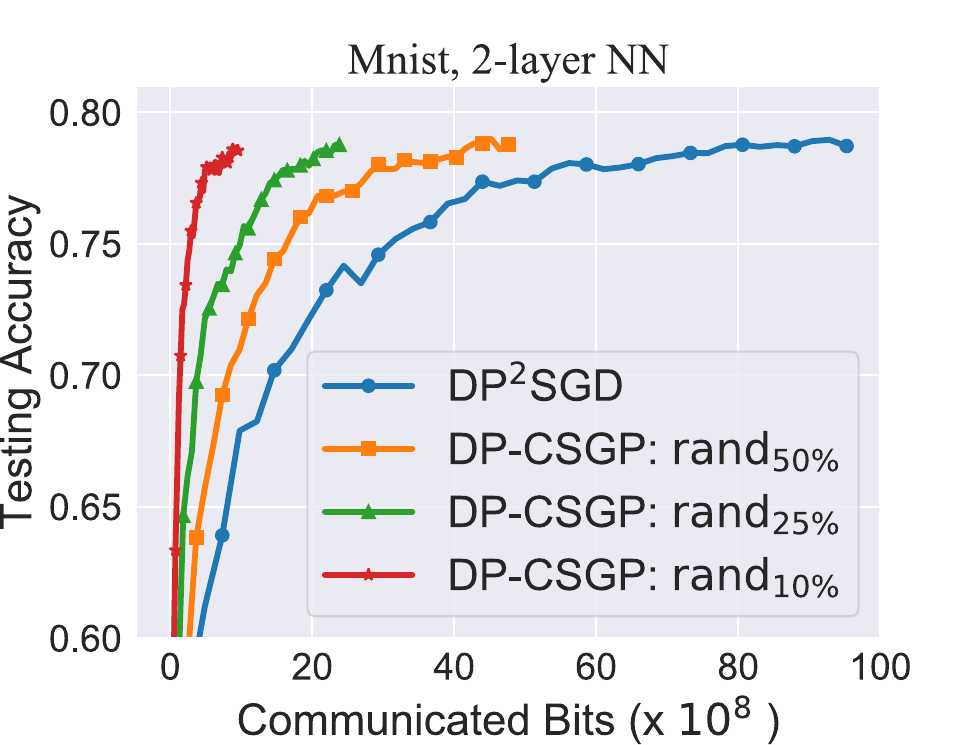}
\label{mnist_sparsify_acc_epsilon_0_point_3}
}
\subfloat[$(0.2, 10^{-4})$-DP]{
\includegraphics[width=0.24\linewidth]{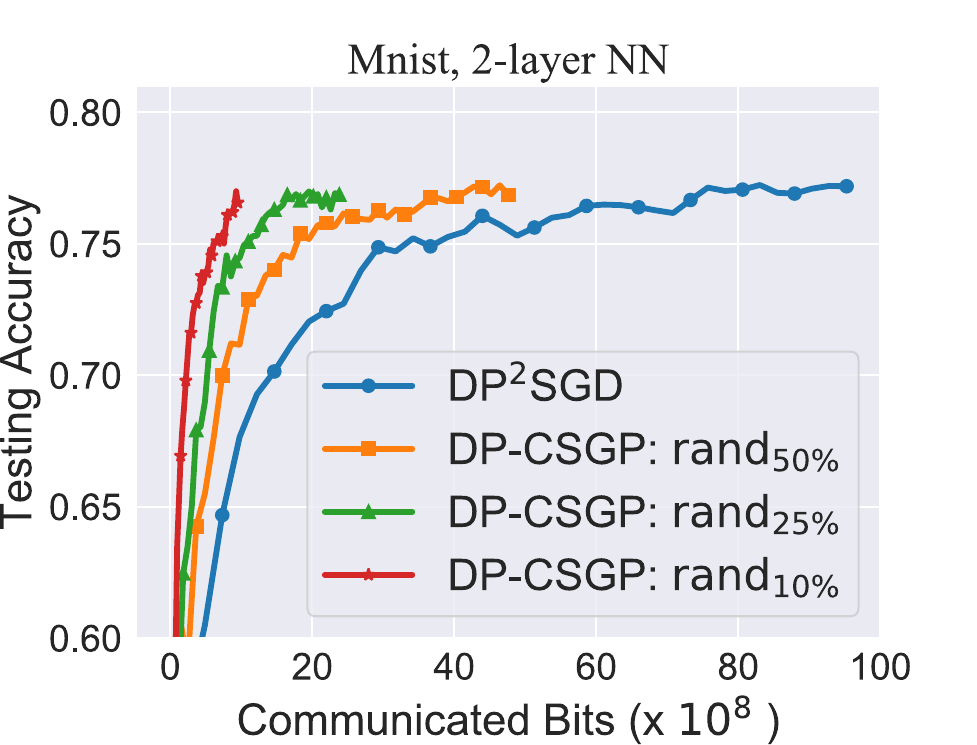}
\label{mnist_sparsify_acc_epsilon_0_point_2}
}
\subfloat[DP-CSGP: rand$_{75\%}$]{
\includegraphics[width=0.24\linewidth]{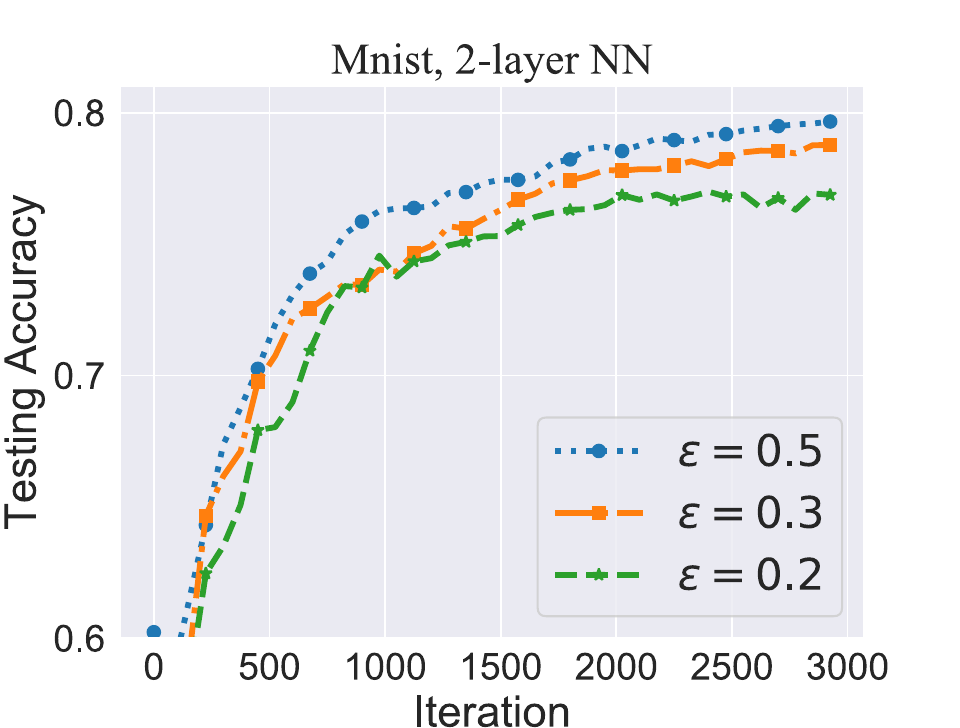}
\label{mnist_sparsify_acc_epsilons}
}
\caption{Convergence performance of our DP-CSGP using $\mathrm{rand_a}$-sparsification with different values of $\mathrm{a}$ and DP$^2$-SGD, when training 2-layer neural network on Mnist dataset under different privacy budgets.}
\label{Mnist_sparsification}
\end{figure*}

\begin{figure*}[!h]
\centering
\subfloat[$(0.5, 10^{-4})$-DP]{
\includegraphics[width=0.24\linewidth]{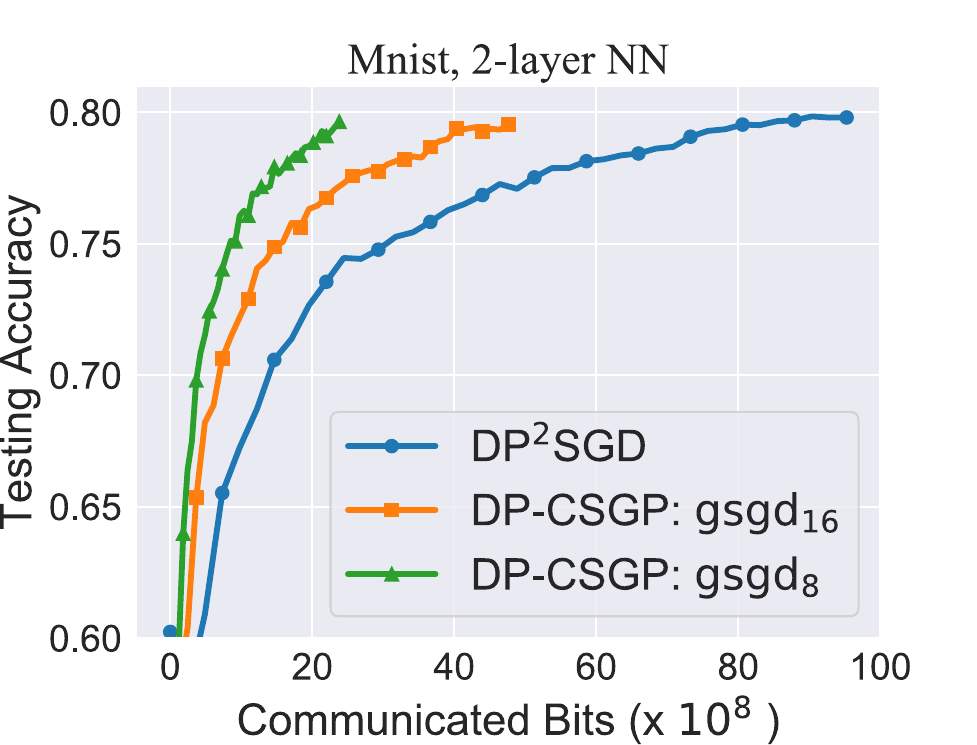}
\label{mnist_quantify_acc_epsilon_0_point_5}
}
\subfloat[$(0.3, 10^{-4})$-DP]{
\includegraphics[width=0.24\linewidth]{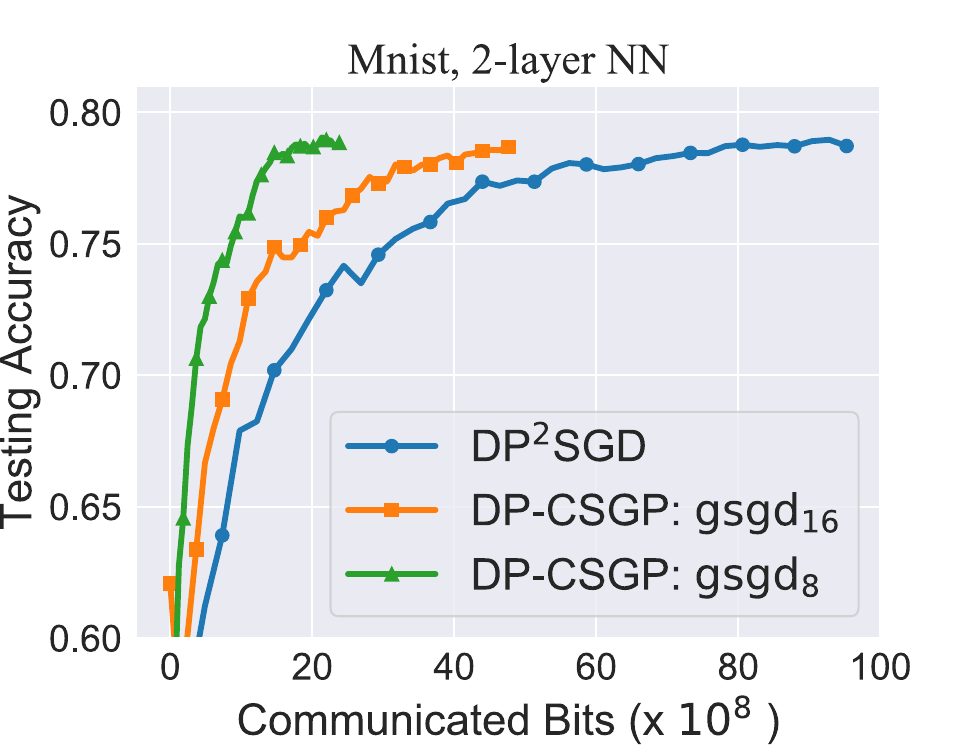}
\label{mnist_quantify_acc_epsilon_0_point_3}
}
\subfloat[$(0.2, 10^{-4})$-DP]{
\includegraphics[width=0.24\linewidth]{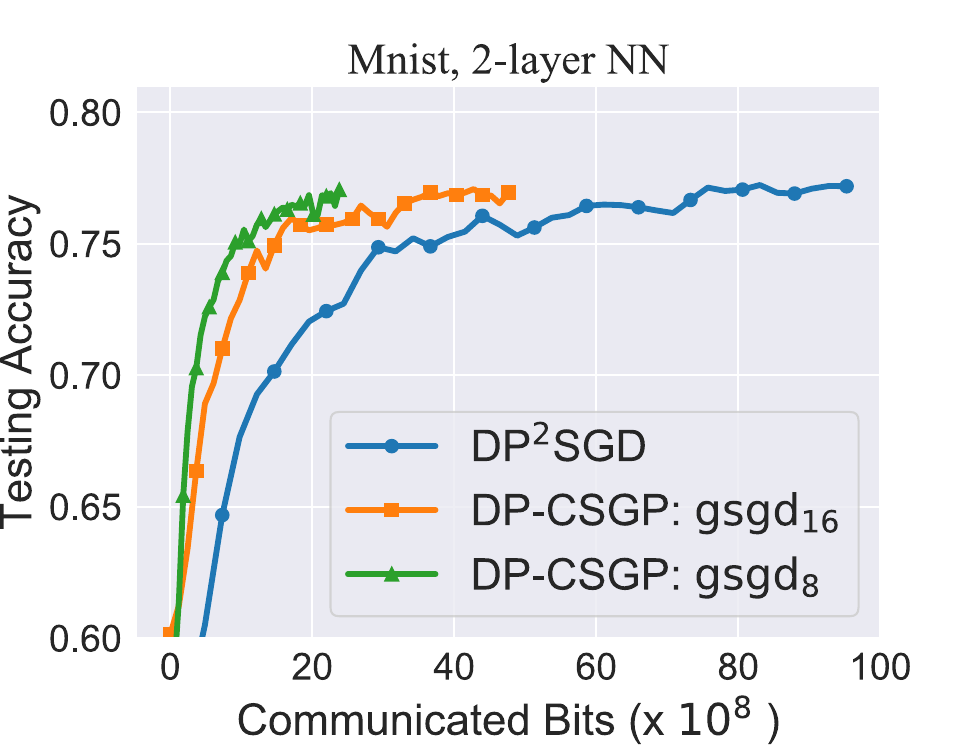}
\label{mnist_quantify_acc_epsilon_0_point_2}
}
\subfloat[DP-CSGP: gsgd$_8$]{
\includegraphics[width=0.24\linewidth]{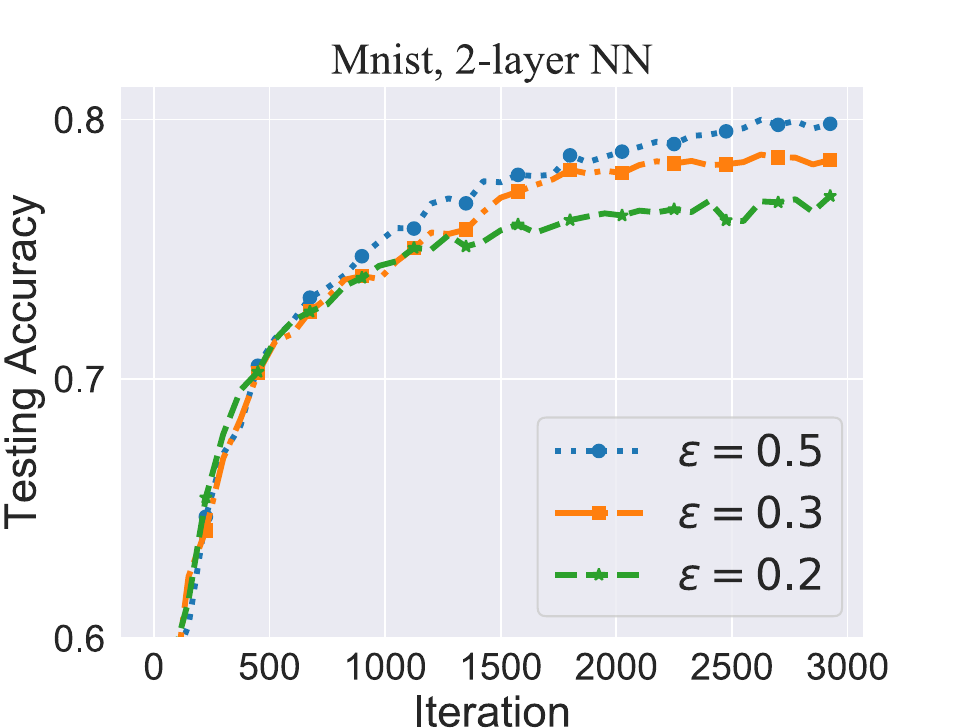}
\label{mnist_quantify_acc_epsilons}
}
\caption{Convergence performance of our DP-CSGP using $\mathrm{gsgd_b}$-quantification with different values of $\mathrm{b}$ and DP$^2$-SGD, when training 2-layer neural network on Mnist dataset under different privacy budgets.}
\label{Mnist_quantification}
\end{figure*}

\begin{figure*}[!h]
\centering
\subfloat[$(10, 10^{-4})$-DP]{
\includegraphics[width=0.242\linewidth]{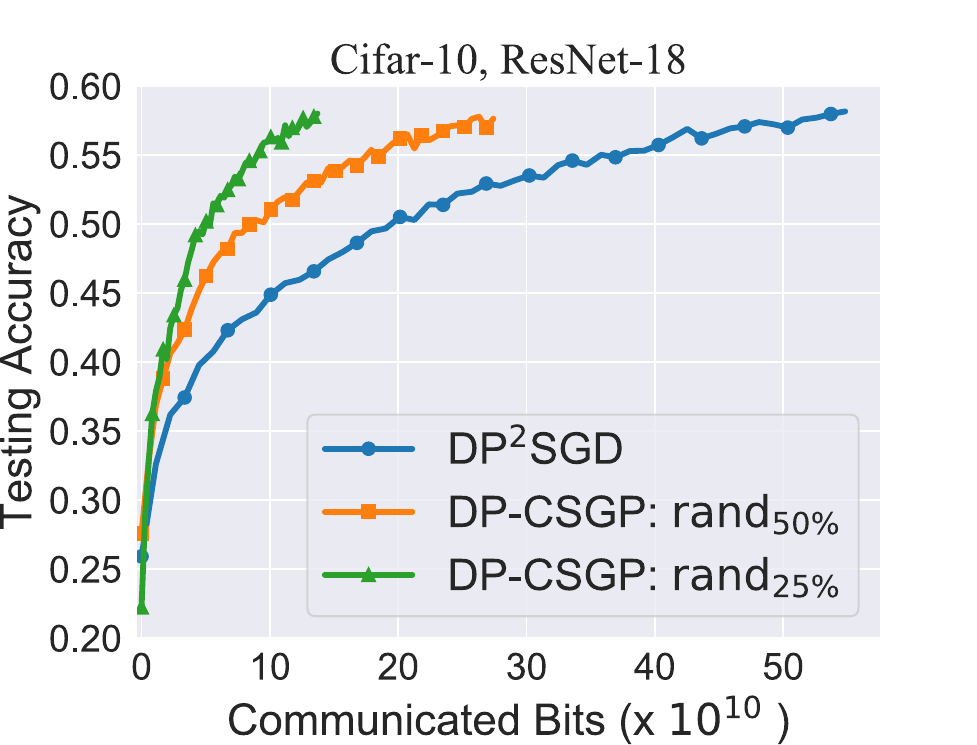}
\label{cifar_sparsify_acc_epsilon_10}
}
\subfloat[$(3, 10^{-4})$-DP]{
\includegraphics[width=0.242\linewidth]{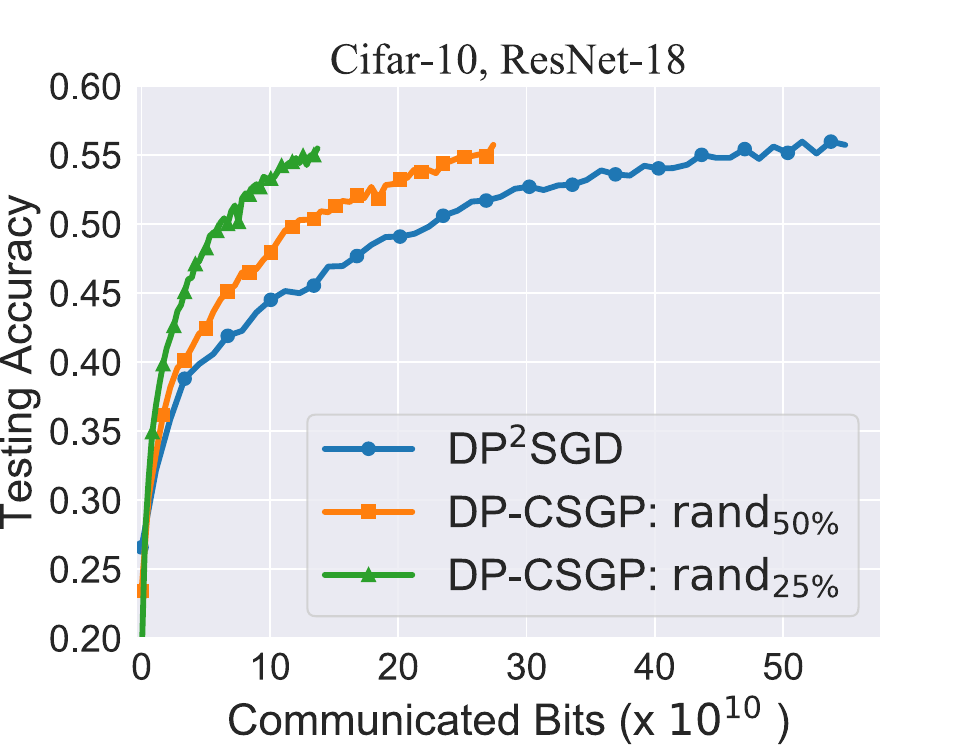}
\label{cifar_sparsify_acc_epsilon_3}
}
\subfloat[$(1, 10^{-4})$-DP]{
\includegraphics[width=0.242\linewidth]{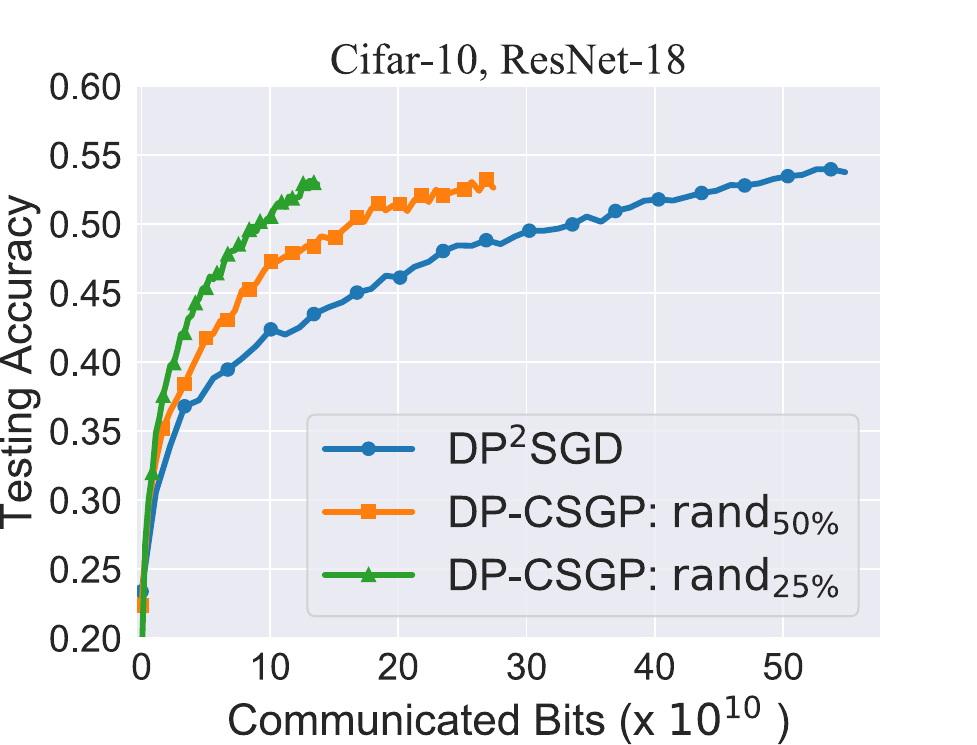}
\label{cifar_sparsify_acc_epsilon_1}
}
\subfloat[DP-CSGP: rand$_{75\%}$]{
\includegraphics[width=0.23\linewidth]{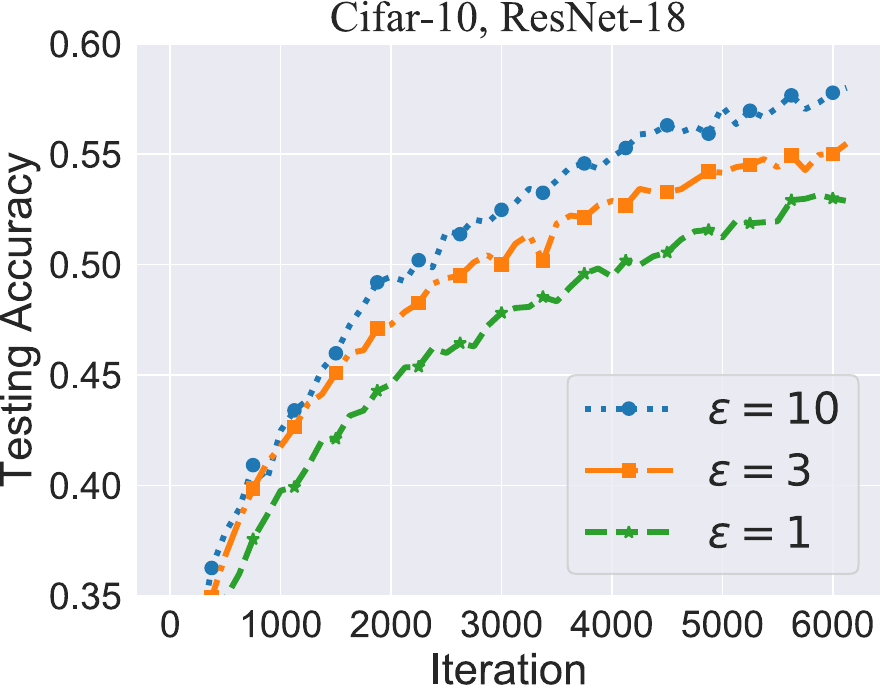}
\label{cifar_sparsify_acc_epsilons}
}
\caption{Convergence performance of our DP-CSGP using $\mathrm{rand_a}$-sparsification with different values of $\mathrm{a}$ and DP$^2$-SGD, when training ResNet-18 on Cifar-10 dataset under different privacy budgets.}
\label{Cifar_sparsification}
\end{figure*}

\begin{figure*}[!h]
\centering
\subfloat[$(10, 10^{-4})$-DP]{
\includegraphics[width=0.24\linewidth]{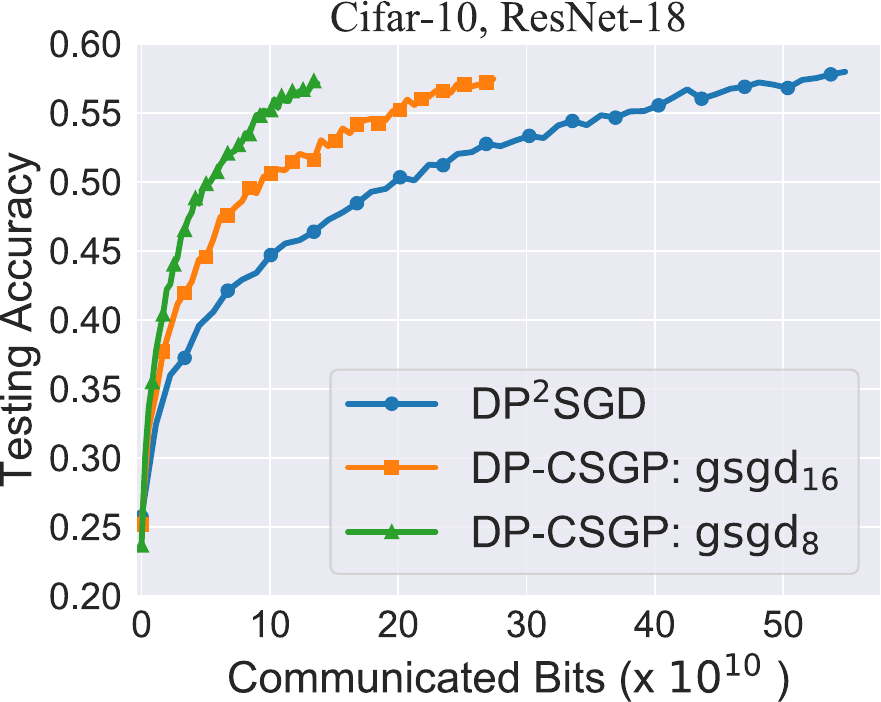}
\label{cifar_quantify_acc_epsilon_10}
}
\subfloat[$(3, 10^{-4})$-DP]{
\includegraphics[width=0.24\linewidth]{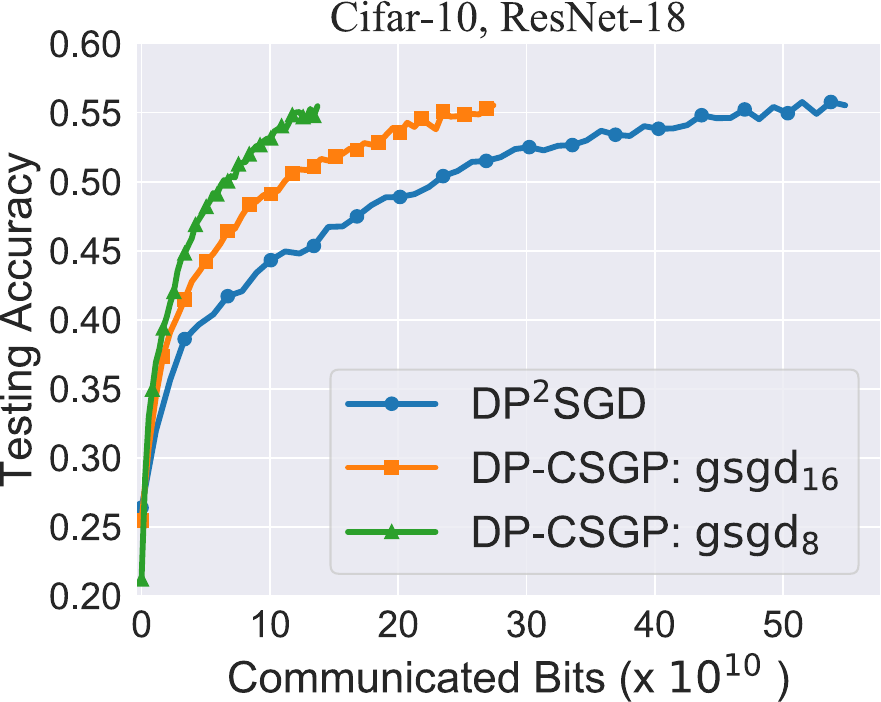}
\label{cifar_quantify_acc_epsilon_3}
}
\subfloat[$(1, 10^{-4})$-DP]{
\includegraphics[width=0.24\linewidth]{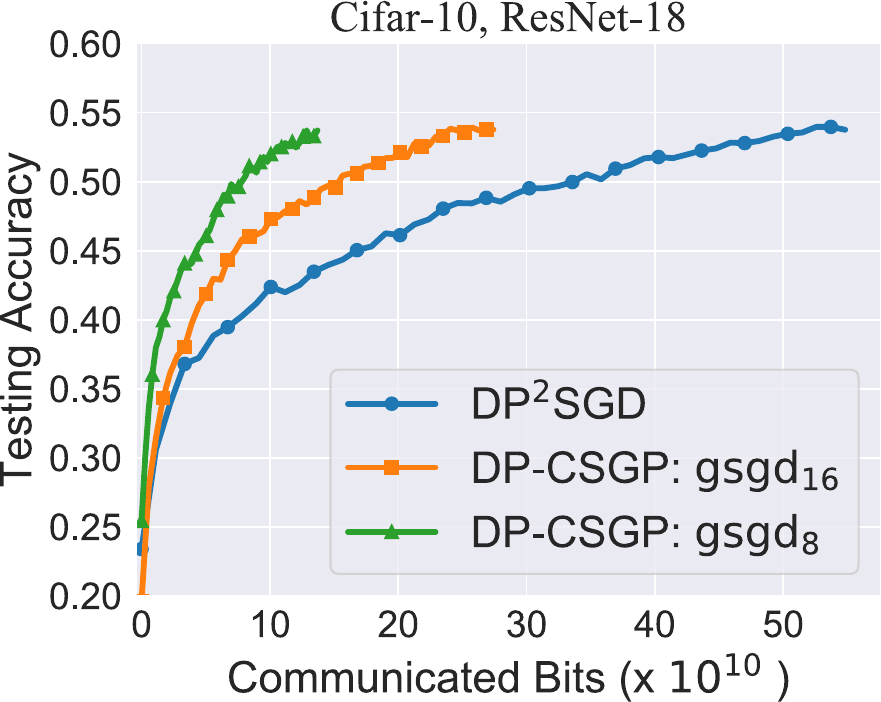}
\label{cifar_quantify_acc_epsilon_1}
}
\subfloat[DP-CSGP: gsgd$_8$]{
\includegraphics[width=0.24\linewidth]{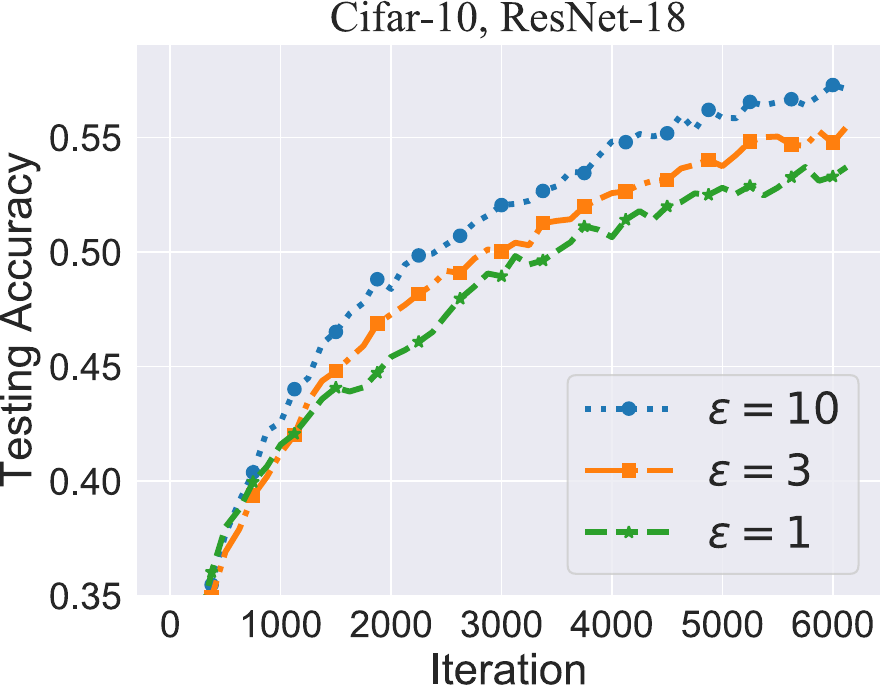}
\label{cifar_quantify_acc_epsilons}
}
\caption{Convergence performance of our DP-CSGP using $\mathrm{gsgd_b}$-quantification with different values of $\mathrm{a}$ and DP$^2$-SGD, when training ResNet-18 on Cifar-10 dataset under different privacy budgets.}
\label{Cifar_quantification}
\end{figure*}

For the ResNet-18 training task, we consider three levels of privacy with $\epsilon=10, 3, 1$ and a common $\delta=10^{-4}$. For our {\alg}, test different levels of sparsification operator ($\mathrm{rand}_{50}$, $\mathrm{rand}_{75}$) and different levels of quantification operator ($\mathrm{gsgd}_{16}$, $\mathrm{gsgd}_{8}$), where the experimental results are reported in Fig.~\ref{Cifar_sparsification} and~\ref{Cifar_quantification} respectively.
The takeaways from the experimental results are similar to previous experiments on 2-layer neural network training task. In terms of communication bits, our {\alg} again performs much better than the uncompressed counterpart DP$^2$SGD while maintaining the same level of privacy protection, again emphasizing {\alg}'s communication efficiency thanks to the employment of communication compression schemes.
Furthermore, a comparison of performance for {\alg} with different values of privacy budgets $\epsilon$ (c.f., Fig.~\ref{cifar_sparsify_acc_epsilons} and~\ref{cifar_quantify_acc_epsilons}) shows that the stronger the level of required privacy protection (i.e., the smaller the value of budget $\epsilon$), the lower the model accuracy, which again highlights the utility-privacy trade-off of {\alg}.

\section{Conclusion}
In this work, we have proposed a differentially private decentralized learning method {\alg}, which employs an error-feedback-based communication compression mechanism to improve communication efficiency. We have established a tight utility bound of $\mathcal{O}\left( \sqrt{d\log \left( \frac{1}{\delta} \right)}/(\sqrt{n}J\epsilon) \right)$ for the proposed {\alg} under mild assumptions. This utility bound matches that of differentially private decentralized learning methods with exact communication, demonstrating that our {\alg} can maintain strong model utility while ensuring both differential privacy and communication efficiency. Extensive experiments are conducted to demonstrate the superiority of our proposed algorithm, in fully decentralized settings.

\appendix
Here we supplement the proof of upper bounding the consensus error $\mathbb{E}\left[ \left\| z_{i}^{t+1}- \bar{x}^t \right\| ^2 \right]$ as below.

Based on the iterates in~\eqref{concatenate_ierate},  we can rewrite the update rule for $X^{t+1}$ as
\begin{equation*}
\begin{aligned}
X^{t+1}= & AX^t+\left( A-I \right) \left( \hat{X}^{t+1}-X^t \right)
\\
 &-\eta \cdot \left( \partial F\left( Z^{t+1};\xi ^{t+1} \right) +N^{t+1} \right) .
\end{aligned}
\end{equation*}
By repeating the above for $X^{t},...X^{1}$, we have
\begin{equation}
\label{ref_1}
\begin{aligned}
X^{t+1}=&A^tX^1+\sum_{s=0}^{t-1}{A^s \left( A-I \right) \left( \hat{X}^{t-s+1}-X^{t-s} \right)}
\\
& -\eta \cdot \sum_{s=0}^{t-1}{A^s \left( \partial F\left( Z^{t-s+1};\xi ^{t-s+1} \right) +N^{t-s+1} \right)}.
\end{aligned}
\end{equation}
Multiplying $\mathbf{1}^\top$ on both sides of~\eqref{ref_1}, yields
\begin{equation}
\begin{aligned}
\mathbf{1}^{\top}X^{t+1}= & \mathbf{1}^{\top}X^1-\eta \cdot \sum_{s=0}^{t-1}{\mathbf{1}^{\top}\partial F\left( Z^{t-s+1};\xi ^{t-s+1} \right)}
\\
& -\eta \cdot \sum_{s=0}^{t-1}{\mathbf{1}^{\top}N^{t-s+1}},
\end{aligned}
\end{equation}
where we used $\mathbf{1}^{\top}A=\mathbf{1}^{\top}
$ in Assumption~\ref{Ass_weight_matrix}.

Based on the above two inequalities, we have
\begin{equation}
\begin{aligned}
&\left\| X^{t+1}-\phi \mathbf{1}^{\top}X^t \right\| 
\\
\leqslant & \left\| \left( A^t-\phi \mathbf{1}^{\top} \right) X^1 \right\| +\eta \left\| \sum_{s=0}^{t-1}{\left( A^s-\phi \mathbf{1}^{\top} \right) N^{t-s+1}} \right\| 
\\
&+\left\| \sum_{s=0}^{t-1}{\left( A^s-\phi \mathbf{1}^{\top} \right) \left( A-I \right) \left( \hat{X}^{t-s+1}-X^{t-s} \right)} \right\| 
\\
&+\eta \left\| \sum_{s=0}^{t-1}{\left( A^s-\phi \mathbf{1}^{\top} \right) \partial F\left( Z^{t-s+1};\xi ^{t-s+1} \right)} \right\| 
\\
\leqslant & C\lambda ^t\left\| X^1 \right\| +C\sum_{s=0}^{t-1}{\lambda ^s\left\| \hat{X}^{t-s+1}-X^{t-s} \right\|}
\\
&+\eta C\sum_{s=0}^{t-1}{\lambda ^s\left\| \partial F\left( Z^{t-s+1};\xi ^{t-s+1} \right) \right\|}
\\
& +\eta C\sum_{s=0}^{t-1}{\lambda ^s\left\| N^{t-s+1} \right\|},
\end{aligned}
\end{equation}
where we used~\eqref{network_convergence} in the second inequality.

Using $\left( a+b+c+d \right) ^2\leqslant 4a^2+4b^2+4c^2+4d^2$, we have
\begin{equation}
\label{ref_2}
\begin{aligned}
&\left\| X^{t+1}-\phi \mathbf{1}^{\top}X^t \right\| ^2
\\
\leqslant & 4C^2\left( \sum_{s=0}^{t-1}{\lambda ^s\left\| \hat{X}^{t-s+1}-X^{t-s} \right\|} \right) ^2
\\
& +4\eta ^2C^2\left( \sum_{s=0}^{t-1}{\lambda ^s\left\| \partial F\left( Z^{t-s+1};\xi ^{t-s+1} \right) \right\|} \right) ^2
\\
& +4\eta ^2C^2\left( \sum_{s=0}^{t-1}{\lambda ^s\left\| N^{t-s+1} \right\|} \right) ^2+4C^2\lambda ^{2t}\left\| X^1 \right\| ^2.
\end{aligned}
\end{equation}
For the first term in the RHS of~\eqref{ref_2}, we have
\begin{equation}
\label{bound_1}
\begin{aligned}
& \left( \sum_{s=0}^{t-1}{\lambda ^s\left\| \hat{X}^{t-s+1}-X^{t-s} \right\|} \right) ^2
\\
= & \left( \sum_{s=0}^{t-1}{\lambda ^{\frac{s}{2}}\cdot \left( \lambda ^{\frac{s}{2}}\left\| \hat{X}^{t-s+1}-X^{t-s} \right\| \right)} \right) ^2
\\ 
\overset{\left( a \right)}{\leqslant} & \sum_{s=0}^{t-1}{\left( \lambda ^{\frac{s}{2}} \right) ^2}\cdot \sum_{s=0}^{t-1}{\left( \lambda ^{\frac{s}{2}}\left\| \hat{X}^{t-s+1}-X^{t-s} \right\| \right) ^2}
\\
\leqslant & \frac{1}{1-\lambda}\sum_{s=0}^{t-1}{\lambda ^s\left\| \hat{X}^{t-s+1}-X^{t-s} \right\| ^2},
\end{aligned}
\end{equation}
where we used Cauchy-Swarchz inequality in $(a)$.

Using a similar approach as above, we can bound the second term and the third term in the RHS of~\eqref{ref_2} as
\begin{equation}
\label{bound_2}
\begin{aligned}
& \left( \sum_{s=0}^{t-1}{\lambda ^s\left\| \partial F\left( Z^{t-s+1};\xi ^{t-s+1} \right) \right\|} \right) ^2
\\
\leqslant & \frac{1}{1-\lambda}\sum_{s=0}^{t-1}{\lambda ^s\left\| \partial F\left( Z^{t-s+1};\xi ^{t-s+1} \right) \right\| ^2}
\end{aligned}
\end{equation}
and 
\begin{equation}
\label{bound_3}
\left( \sum_{s=0}^{t-1}{\lambda ^s\left\| N^{t-s+1} \right\|} \right) ^2\leqslant \frac{1}{1-\lambda}\sum_{s=0}^{t-1}{\lambda ^s\left\| N^{t-s+1} \right\| ^2}.
\end{equation}
Substituting~\eqref{bound_1},~\eqref{bound_2} and~\eqref{bound_3} into~\eqref{ref_2}, and taking expectations on both sides, we have
\begin{equation}
\label{upper_bound_1}
\begin{aligned}
& \mathbb{E} \left[ \left\| X^{t+1}-\phi \mathbf{1}^{\top}X^t \right\| ^2\right]
\\
\leqslant & \frac{4C^2}{1-\lambda}\sum_{s=0}^{t-1}{\lambda ^s\mathbb{E}\left[ \left\| \hat{X}^{t-s+1}-X^{t-s} \right\| ^2 \right]}
\\
& +\frac{4\eta ^2C^2}{1-\lambda}\sum_{s=0}^{t-1}{\lambda ^s\mathbb{E}\left[ \left\| \partial F\left( Z^{t-s+1};\xi ^{t-s+1} \right) \right\| ^2 \right]}
\\
& +\frac{4\eta ^2C^2}{1-\lambda}\sum_{s=0}^{t-1}{\lambda ^s\mathbb{E}\left[ \left\| N^{t-s+1} \right\| ^2 \right]}+4C^2\lambda ^{2t}\left\| X^1 \right\| ^2
\\
\leqslant & \frac{4C^2}{1-\lambda}\sum_{s=0}^{t-1}{\lambda ^s\mathbb{E}\left[ \left\| \hat{X}^{t-s+1}-X^{t-s} \right\| ^2 \right]}
\\
& +\frac{4\eta ^2C^2n\left( G^2+ d \sigma ^2 \right)}{\left( 1-\lambda \right) ^2},
\end{aligned}
\end{equation}
where we used Assumption~\ref{assumption_initialization} and~\ref{assumption_bounded_gradient} in the last inequality.

Now, we bound $\mathbb{E}\left[ \left\| X^{t+1}-\hat{X}^{t+2} \right\| ^2 \right] $ as
\begin{equation}
\label{eq_1}
\begin{aligned}
& \mathbb{E}\left[ \left\| X^{t+1}-\hat{X}^{t+2} \right\| ^2 \right] 
\\
\overset{\eqref{update_2}}{=} &  \mathbb{E}\left[ \left\| X^{t+1}-\left( \hat{X}^{t+1}+Q^{t+1} \right) \right\| ^2 \right] 
\\
\overset{\eqref{update_1}}{=} & \mathbb{E}\left[ \left\| X^{t+1}-\hat{X}^{t+1}-Q\left( X^{t+1}-\hat{X}^{t+1} \right) \right\| ^2 \right] 
\\
\leqslant &  \omega ^2 \mathbb{E} \left[ \left\| X^{t+1}-\hat{X}^{t+1} \right\| ^2 \right] ,
\end{aligned}
\end{equation}
where we used Assumption~\ref{quantization_ratio} in the inequality.

Combining~\eqref{update_2} and~\eqref{update_6}, we have
\begin{equation}
\begin{aligned}
X^{t+1}= &X^t+\left( A-I \right) \hat{X}^{t+1}
\\
& -\eta \cdot \left( \partial F\left( Z^{t+1};\xi ^{t+1} \right) +N^{t+1} \right).
\end{aligned}
\end{equation}
Substituting the above into~\eqref{eq_1} yields
\begin{equation*}
\begin{aligned}
&\mathbb{E}\left[ \left\| X^{t+1}-\hat{X}^{t+2} \right\| ^2 \right] 
\\
\leqslant & 5\omega ^2\mathbb{E}\left[ \left\| X^t-\hat{X}^{t+1} \right\| ^2 \right] +5\omega ^2\mathbb{E}\left[ \left\| \left( A-I \right) \left( \hat{X}^{t+1}-X^t \right) \right\| ^2 \right] 
\\
&+5\omega ^2\mathbb{E}\left[ \left\| \left( A-I \right) \left( X^t-\phi \mathbf{1}^{\top}X^{t-1} \right) \right\| ^2 \right] 
\\
&+5\omega ^2\eta ^2\mathbb{E}\left[ \left\| \partial F\left( Z^{t+1};\xi ^{t+1} \right) \right\| ^2 \right] +5\omega ^2\eta ^2\mathbb{E}\left[ \left\| N^{t+1} \right\| ^2 \right] 
\\
\leqslant & 5\omega ^2\mathbb{E}\left[ \left\| X^t-\hat{X}^{t+1} \right\| ^2 \right] +5\omega ^2\left\| \left( A-I \right) \right\| ^2\mathbb{E}\left[ \left\| \hat{X}^{t+1}-X^t \right\| ^2 \right] 
\\
&+5\omega ^2\left\| \left( A-I \right) \right\| ^2\mathbb{E}\left[ \left\| X^t-\phi \mathbf{1}^{\top}X^{t-1} \right\| ^2 \right] 
\\
& +5n\omega ^2\eta ^2G^2+5n\omega ^2  \eta ^2 d \sigma ^2,
\end{aligned}
\end{equation*}
where we used Assumption~\ref{assumption_bounded_gradient} in the last inequality.

Let $\gamma \triangleq  \| A-I \|$, we have
\begin{equation}
\label{eq_3}
\begin{aligned}
& \mathbb{E}\left[ \left\| X^{t+1}-\hat{X}^{t+2} \right\| ^2 \right] 
\\
\leqslant & 5\omega ^2\left( 1+\gamma ^2 \right) \mathbb{E}\left[ \left\| X^t-\hat{X}^{t+1} \right\| ^2 \right] +5n\omega ^2\eta ^2G^2
\\
& +5\omega ^2\gamma ^2\mathbb{E}\left[ \left\| X^t-\phi \mathbf{1}^{\top}X^{t-1} \right\| ^2 \right] +5n\omega ^2\eta ^2 d \sigma ^2.
\end{aligned}
\end{equation}
Let $\rho \triangleq  \omega ^2\left( 1+\gamma ^2 \right) $, the inequality~\eqref{eq_3} can be relaxed as
\begin{equation}
\label{upper_bound_2}
\begin{aligned}
& \mathbb{E}\left[ \left\| X^{t+1}-\hat{X}^{t+2} \right\| ^2 \right] 
\\
\leqslant & 5\rho \left( \mathbb{E}\left[ \left\| X^t-\hat{X}^{t+1} \right\| ^2 \right] +n\eta ^2\left( G^2+ d \sigma ^2 \right) \right. 
\\
& +\left. \mathbb{E}\left[ \left\| X^t-\phi \mathbf{1}^{\top}X^{t-1} \right\| ^2 \right] \right).
\end{aligned}
\end{equation}
Define
\begin{equation}
R^{t+1}\triangleq \mathbb{E}\left[ \left\| X^{t+1}-\phi \mathbf{1}^{\top}X^t \right\| ^2 \right] 
\end{equation}
and
\begin{equation}
U^{t+1}\triangleq \mathbb{E}\left[ \left\| X^{t+1}-\hat{X}^{t+2} \right\| ^2 \right].
\end{equation}
Then, \eqref{upper_bound_1} and~\eqref{upper_bound_2} can be rewritten as
\begin{equation}
\label{combine_bound}
\left\{ \begin{array}{c}
	R^{t+1}\leqslant \frac{4C^2}{1-\lambda}\sum_{s=0}^{t-1}{\lambda ^sU^{t-s}}+\frac{4C^2n\eta ^2\left( G^2+ d \sigma ^2 \right)}{\left( 1-\lambda \right) ^2}\\
	U^{t+1}\leqslant 5\rho \left( U^t+R^t+n\eta ^2\left( G^2+d\sigma ^2 \right) \right)\\
\end{array} \right. .
\end{equation}
Next, we show that the quantization error (i.e., $U^t$) decays proportionately with $\eta^2$. Under Assumption~\ref{assumption_initialization}, if $\rho$ satisfies
\begin{equation}
\label{rho_upper_bound}
\rho \leqslant \left( 10+\frac{40C^2}{\left( 1-\lambda \right) ^2} \right) ^{-1},
\end{equation}
the inequalities in~\eqref{combine_bound} satisfy
\begin{equation}
\label{eq_4}
U^t\leqslant \zeta \eta ^2,~\forall t \geqslant 1
\end{equation}
where 
\begin{equation}
\label{def_zeta}
\zeta =10\rho \left( n\left( G^2+d\sigma ^2 \right) +\frac{4C^2n\left( G^2+d\sigma ^2 \right)}{\left( 1-\lambda \right) ^2} \right) .
\end{equation}
We now show that Eq.~\eqref{eq_4} holds by induction. First, we can write the inequalities in~\eqref{combine_bound} based on $U$ to obtain
\begin{equation}
\begin{aligned}
& U^{t+1}
\\
\leqslant & 5\rho \left( U^t+n\eta ^2\left( G^2+d\sigma ^2 \right) \right) 
\\
& +5\rho \left( \frac{4C^2}{1-\lambda}\sum_{s=0}^{t-2}{\lambda ^sU^{t-s-1}}+\frac{4C^2n\eta ^2\left( G^2+d\sigma ^2 \right)}{\left( 1-\lambda \right) ^2} \right).
\end{aligned}
\end{equation}
Then, supposing Eq.~\eqref{eq_4} holds, we can further obtain the following for $U^{t+1}$:
\begin{equation}
\label{eq_5}
\begin{aligned}
& U^{t+1}
\\
\leqslant & 5\rho \left( \zeta \eta ^2+n\eta ^2\left( G^2+d\sigma ^2 \right) \right) 
\\
& +5\rho \left( \frac{4C^2\zeta \eta ^2}{1-\lambda}\sum_{s=0}^{t-2}{\lambda ^s}+\frac{4C^2n\eta ^2\left( G^2+d\sigma ^2 \right)}{\left( 1-\lambda \right) ^2} \right) 
\\
\leqslant & 5\rho \zeta \eta ^2\left( 1+\frac{4C^2}{\left( 1-\lambda \right) ^2} \right) 
\\
& +5\rho \eta ^2\left( n\left( G^2+d\sigma ^2 \right) +\frac{4C^2n\left( G^2+d\sigma ^2 \right)}{\left( 1-\lambda \right) ^2} \right) .
\end{aligned}
\end{equation}
Invoking $\zeta =10\rho \left( n\left( G^2+d\sigma ^2 \right) +\frac{4C^2n\left( G^2+d\sigma ^2 \right)}{\left( 1-\lambda \right) ^2} \right) $ and $\rho \leqslant \left( 10+\frac{40C^2}{\left( 1-\lambda \right) ^2} \right) ^{-1}$ into~\eqref{eq_5}, we have
\begin{equation}
U^{t+1}\leqslant \frac{\zeta \eta ^2}{2}+\frac{\zeta \eta ^2}{2}=\zeta \eta ^2.
\end{equation}
Note that based on the iterations of the algorithm and Assumption~\ref{assumption_initialization}, we can conclude that $U^1 = 0$. Therefore, Eq.~\eqref{eq_4} holds for all $t \geqslant 1$.

Now we move on the proof of upper bounding the consensus error $\mathbb{E}\left[ \left\| z_{i}^{t+1}- \bar{x}^t \right\| ^2 \right]$. From the update rule in~\eqref{update_4}, we obtain the following for all $t \geqslant 1$:
\begin{equation}
y^{t+1}=Ay^t=A^ty^1.
\end{equation}
According to $y^1=\mathbf{1}$, it yields that
\begin{equation}
\begin{aligned}
y^{t+1} & =A^t\mathbf{1}=\left( A^t-\phi \mathbf{1}^{\top} \right) \mathbf{1}+\phi \mathbf{1}^{\top}\mathbf{1}
\\
& =\left( A^t-\phi \mathbf{1}^{\top} \right) \mathbf{1}+n\phi .
\end{aligned}
\end{equation}
Therefore, for all $i$, we have
\begin{equation}
y_{i}^{t+1}=\left[ \left( A^t-\phi \mathbf{1}^{\top} \right) \mathbf{1} \right] _i+n\phi _i.
\end{equation}
Furthermore, based on~\eqref{update_5}, the de-biased model parameter $z_i^{t+1}$ satisfies
\begin{equation*}
\begin{aligned}
& z_{i}^{t+1}=\frac{x_{i}^{t+1}}{y_{i}^{t+1}}
\\
=& \frac{\left[ \sum_{s=0}^{t-1}{A^s\left( A-I \right) \left( \hat{X}^{t-s+1}-X^{t-s} \right)} \right] _i}{\left[ \left( A^t-\phi \mathbf{1}^{\top} \right) \mathbf{1} \right] _i+n\phi _i}
\\
& -\frac{\left[ \eta \cdot \sum_{s=0}^{t-1}{A^s\left( \partial F\left( Z^{t-s+1};\xi ^{t-s+1} \right) +N^{t-s+1} \right)} \right] _i}{\left[ \left( A^t-\phi \mathbf{1}^{\top} \right) \mathbf{1} \right] _i+n\phi _i}.
\end{aligned}
\end{equation*}
Using this above relation, the vector corresponding to the consensus error of node $i$ can be obtained as
\begin{equation*}
\begin{aligned}
& z_{i}^{t+1}-\frac{\mathbf{1}^{\top}X^t}{n}
\\
= & \frac{\left[ \sum_{s=0}^{t-1}{A^s\left( A-I \right) \left( \hat{X}^{t-s+1}-X^{t-s} \right)} \right] _i}{\left[ \left( A^t-\phi \mathbf{1}^{\top} \right) \mathbf{1} \right] _i+n\phi _i}
\\
& -\frac{\left[ \eta \cdot \sum_{s=0}^{t-1}{A^s\left( \partial F\left( Z^{t-s+1};\xi ^{t-s+1} \right) +N^{t-s+1} \right)} \right] _i}{\left[ \left( A^t-\phi \mathbf{1}^{\top} \right) \mathbf{1} \right] _i+n\phi _i}
\\
& +\left( \left[ \left( A^t-\phi \mathbf{1}^{\top} \right) \mathbf{1} \right] _i+n\phi _i \right) \cdot 
\\
& \frac{\eta \cdot \mathbf{1}^{\top}\sum_{s=0}^{t-1}{\left( \partial F\left( Z^{t-s+1};\xi ^{t-s+1} \right) +N^{t-s+1} \right)}}{n\left( \left[ \left( A^t-\phi \mathbf{1}^{\top} \right) \mathbf{1} \right] _i+n\phi _i \right)}
\\
= & \frac{\sum_{s=0}^{t-1}{\left[ A^s\left( A-I \right) \right] _i\left( \hat{X}^{t-s+1}-X^{t-s} \right)}}{\left[ \left( A^t-\phi \mathbf{1}^{\top} \right) \mathbf{1} \right] _i+n\phi _i}
\\
& -\frac{\eta \sum_{s=0}^{t-1}{\left[ A^s-\phi \mathbf{1}^{\top} \right] _i\left( \partial F\left( Z^{t-s+1};\xi ^{t-s+1} \right) +N^{t-s+1} \right)}}{\left[ \left( A^t-\phi \mathbf{1}^{\top} \right) \mathbf{1} \right] _i+n\phi _i}
\\
& +\eta \left[ \left( A^t-\phi \mathbf{1}^{\top} \right) \mathbf{1} \right] _i\cdot 
\\
& \frac{\sum_{s=0}^{t-1}{\left( \partial F\left( Z^{t-s+1};\xi ^{t-s+1} \right) +N^{t-s+1} \right)}}{n\left( \left[ \left( A^t-\phi \mathbf{1}^{\top} \right) \mathbf{1} \right] _i+n\phi _i \right)}.
\end{aligned}
\end{equation*}
Note that by Proposition~\ref{Pro}, we have $\left[ \left( A^t-\phi \mathbf{1}^{\top} \right) \mathbf{1} \right] _i+n\phi _i=\left[ A^t\mathbf{1} \right] _i\geqslant \beta $ for all $t \geqslant 1$, which yields the following for squared norm of consensus error:
\begin{equation}
\label{eq_6}
\begin{aligned}
& \mathbb{E}\left[ \left\| z_{i}^{t+1}-\frac{\mathbf{1}^{\top}X^t}{n} \right\| ^2 \right] 
\\
& \leqslant  \frac{5}{\beta ^2}\mathbb{E}\left[ \left\| \sum_{s=0}^{t-1}{\left[ A^s\left( A-I \right) \right] _i\left( \hat{X}^{t-s+1}-X^{t-s} \right)} \right\| ^2 \right] 
\\
& +\frac{5\eta ^2}{\beta ^2}\mathbb{E}\left[ \left\| \sum_{s=0}^{t-1}{\left[ A^s-\phi \mathbf{1}^{\top} \right] _i\cdot \partial F\left( Z^{t-s+1};\xi ^{t-s+1} \right)} \right\| ^2 \right] 
\\
& +\frac{5\eta ^2}{\beta ^2}\mathbb{E}\left[ \left\| \sum_{s=0}^{t-1}{\left[ A^s-\phi \mathbf{1}^{\top} \right] _i\cdot N^{t-s+1}} \right\| ^2 \right] 
\\
& +\frac{5\eta ^2}{n^2\beta ^2}\mathbb{E}\left[ \left\| \mathbf{1}^{\top}\sum_{s=0}^{t-1}{\left[ \left( A^t-\phi \mathbf{1}^{\top} \right) \mathbf{1} \right] _i\partial F\left( Z^{t-s+1};\xi ^{t-s+1} \right)} \right\| ^2 \right] 
\\
& +\frac{5\eta ^2}{n^2\beta ^2}\mathbb{E}\left[ \left\| \mathbf{1}^{\top}\sum_{s=0}^{t-1}{\left[ \left( A^t-\phi \mathbf{1}^{\top} \right) \mathbf{1} \right] _i\cdot N^{t-s+1}} \right\| ^2 \right] .
\end{aligned}
\end{equation}
By expanding the first term in the RHS of~\eqref{eq_6}, we have
\begin{equation*}
\begin{aligned}
& \left\| \sum_{s=0}^{t-1}{\left[ A^s\left( A-I \right) \right] _i\left( \hat{X}^{t-s+1}-X^{t-s} \right)} \right\| ^2
\\
= & \sum_{s=0}^{t-1}{\left\| \left[ A^s\left( A-I \right) \right] _i\left( \hat{X}^{t-s+1}-X^{t-s} \right) \right\| ^2}
\\
&  +\sum_{s\ne s^{\prime}}^{t-1}{\left< \left[ A^s\left( A-I \right) \right] _i\left( \hat{X}^{t-s+1}-X^{t-s} \right) , \right.}
\\
&  \left. \left[ A^{s^{\prime}}\left( A-I \right) \right] _i\left( \hat{X}^{t-s^{\prime}+1}-X^{t-s^{\prime}} \right) \right> 
\\
\leqslant &  \sum_{s=0}^{t-1}{\left\| \left[ A^s\left( A-I \right) \right] _i \right\| ^2\left\| \hat{X}^{t-s+1}-X^{t-s} \right\| ^2}
\\
&  +\sum_{s\ne s^{\prime}}^{t-1}{\left\| \left[ A^s\left( A-I \right) \right] _i \right\| \left\| \hat{X}^{t-s+1}-X^{t-s} \right\| \cdot}
\\
&  \left\| \left[ A^{s^{\prime}}\left( A-I \right) \right] _i \right\| \left\| \hat{X}^{t-s^{\prime}+1}-X^{t-s^{\prime}} \right\| .
\end{aligned}
\end{equation*}
Using the fact that $x\cdot y\leqslant \frac{x^2}{2}+\frac{y^2}{2}$ for all $x,y\in \mathbb{R}$, the above inequality reduces to
\begin{equation*}
\begin{aligned}
& \left\| \sum_{s=0}^{t-1}{\left[ A^s\left( A-I \right) \right] _i\left( \hat{X}^{t-s+1}-X^{t-s} \right)} \right\| ^2
\\
\leqslant & \sum_{s=0}^{t-1}{\left\| \left[ A^s\left( A-I \right) \right] _i \right\| ^2\left\| \hat{X}^{t-s+1}-X^{t-s} \right\| ^2}
\\
& +\frac{1}{2}\sum_{s\ne s^{\prime}}^{t-1}{\left\| \left[ A^s\left( A-I \right) \right] _i \right\| \left\| \left[ A^{s^{\prime}}\left( A-I \right) \right] _i \right\| \cdot}
\\
& \left( \left\| \hat{X}^{t-s+1}-X^{t-s} \right\| ^2 + \left\| \hat{X}^{t-s^{\prime}+1}-X^{t-s^{\prime}} \right\| ^2 \right) .
\end{aligned}
\end{equation*}
Next we use Proposition~\ref{Pro} to further obtain
\begin{equation*}
\begin{aligned}
& \left\| \sum_{s=0}^{t-1}{\left[ A^s\left( A-I \right) \right] _i\left( \hat{X}^{t-s+1}-X^{t-s} \right)} \right\| ^2
\\
\leqslant &  C^2\sum_{s=0}^{t-1}{\lambda ^{2s}\left\| \hat{X}^{t-s+1}-X^{t-s} \right\| ^2}
\\
&  +C^2\sum_{s\ne s^{\prime}}^{t-1}{\lambda ^{s+s^{\prime}}\left\| \hat{X}^{t-s+1}-X^{t-s} \right\| ^2}
\\
\leqslant &  C^2\sum_{s=0}^{t-1}{\left( \lambda ^{2s}+\frac{\lambda ^s}{1-\lambda} \right) \left\| \hat{X}^{t-s+1}-X^{t-s} \right\| ^2}
\\
\leqslant &  \frac{2C^2}{1-\lambda}\sum_{s=0}^{t-1}{\lambda ^s\left\| \hat{X}^{t-s+1}-X^{t-s} \right\| ^2},
\end{aligned}
\end{equation*}
where we used $\lambda ^{2s}\leqslant \frac{\lambda ^s}{1-\lambda}$ in the last inequality. Using the same approach, we can derive the upper bound for the second term and the third term in the RHS of~\eqref{eq_6}, as follows
\begin{equation*}
\begin{aligned}
& \mathbb{E}\left[ \left\| \sum_{s=0}^{t-1}{\left[ A^s-\phi \mathbf{1}^{\top} \right] _i\cdot \partial F\left( Z^{t-s+1};\xi ^{t-s+1} \right)} \right\| ^2 \right] 
\\
\leqslant & \frac{2C^2}{1-\lambda}\sum_{s=0}^{t-1}{\lambda ^s\mathbb{E}\left[ \left\| \partial F\left( Z^{t-s+1};\xi ^{t-s+1} \right) \right\| ^2 \right]}. 
\end{aligned}
\end{equation*}
and
\begin{equation*}
\begin{aligned}
& \mathbb{E}\left[ \left\| \sum_{s=0}^{t-1}{\left[ A^s-\phi \mathbf{1}^{\top} \right] _i\cdot N^{t-s+1}} \right\| ^2 \right] 
\\
\leqslant & \frac{2C^2}{1-\lambda}\sum_{s=0}^{t-1}{\lambda ^s\mathbb{E}\left[ \left\| N^{t-s+1} \right\| ^2 \right]} .
\end{aligned}
\end{equation*}
To bound the last two terms in the RHS of~\eqref{eq_6}, we use the same method, as well as the fact that $\left\| A^t-\phi \mathbf{1}^{\top} \right\| \leqslant C\lambda ^t\leqslant C\lambda ^s$ for all $s \leqslant t$ to deduce that
\begin{equation*}
\begin{aligned}
& \mathbb{E}\left[ \left\| \mathbf{1}^{\top}\sum_{s=0}^{t-1}{\left[ \left( A^t-\phi \mathbf{1}^{\top} \right) \mathbf{1} \right] _i\partial F\left( Z^{t-s+1};\xi ^{t-s+1} \right)} \right\| ^2 \right] 
\\
\leqslant & \frac{2n^2C^2}{1-\lambda}\sum_{s=0}^{t-1}{\lambda ^s\mathbb{E}\left[ \left\| \partial F\left( Z^{t-s+1};\xi ^{t-s+1} \right) \right\| ^2 \right]}
\end{aligned}
\end{equation*}
and
\begin{equation*}
\begin{aligned}
& \mathbb{E}\left[ \left\| \mathbf{1}^{\top}\sum_{s=0}^{t-1}{\left[ \left( A^t-\phi \mathbf{1}^{\top} \right) \mathbf{1} \right] _i\cdot N^{t-s+1}} \right\| ^2 \right] 
\\
\leqslant & \frac{2n^2C^2}{1-\lambda}\sum_{s=0}^{t-1}{\lambda ^s\mathbb{E}\left[ \left\| N^{t-s+1} \right\| ^2 \right]}
\end{aligned}
\end{equation*}
Substituting the above upper bounds into~\eqref{eq_6}, we have
\begin{equation*}
\begin{aligned}
& \mathbb{E}\left[ \left\| z_{i}^{t+1}-\frac{\mathbf{1}^{\top}X^t}{n} \right\| ^2 \right] 
\\
\leqslant &  \frac{10C^2}{\beta ^2\left( 1-\lambda \right)}\sum_{s=0}^{t-1}{\lambda ^s\mathbb{E}\left[ \left\| \hat{X}^{t-s+1}-X^{t-s} \right\| ^2 \right]}
\\
& +\frac{20\eta ^2C^2}{\beta ^2\left( 1-\lambda \right)}\sum_{s=0}^{t-1}{\lambda ^s\mathbb{E}\left[ \left\| \partial F\left( Z^{t-s+1};\xi ^{t-s+1} \right) \right\| ^2 \right]}
\\
& +\frac{20\eta ^2C^2}{\beta ^2\left( 1-\lambda \right)}\sum_{s=0}^{t-1}{\lambda ^s\mathbb{E}\left[ \left\| N^{t-s+1} \right\| ^2 \right]}
\end{aligned}
\end{equation*}
Noticing that $U^{t-s}=\mathbb{E}\left[ \left\| X^{t-s}-\hat{X}^{t-s+1} \right\| ^2 \right] \leqslant \zeta \eta ^2$ by~\eqref{eq_4}, $\mathbb{E}\left[ \left\| \partial F\left( Z^{t-s+1};\xi ^{t-s+1} \right) \right\| ^2 \right] \leqslant nG^2$ and $\mathbb{E}\left[ \left\| N^{t-s+1} \right\| ^2 \right] \leqslant n d\sigma ^2$, we further have
\begin{equation}
\label{consensus_error_bound}
\begin{aligned}
& \mathbb{E}\left[ \left\| z_{i}^{t+1}-\frac{\mathbf{1}^{\top}X^t}{n} \right\| ^2 \right] 
\\
\leqslant & \frac{10\eta ^2C^2}{\beta ^2\left( 1-\lambda \right) ^2}\left[ 2n\left( G^2+ d\sigma ^2 \right) +\zeta \right] ,
\end{aligned}
\end{equation}
which completes the proof.

\bibliographystyle{IEEEtran}
\bibliography{reference}	
\end{document}